\documentclass[letterpaper]{article} 

\ifdefined\aaaianonymous
    \usepackage[submission]{aaai2026}  
\else
    \usepackage{aaai2026}              
\fi

\usepackage{times}  
\usepackage{helvet}  
\usepackage{courier}  
\usepackage[hyphens]{url}  
\usepackage{graphicx} 
\def\UrlFont{\rm}  
\usepackage{natbib}  
\usepackage{caption} 
\usepackage{algorithm}
\usepackage{algorithmic}
\usepackage{amsmath}
\usepackage{amssymb} 
\usepackage{cite}
\usepackage{enumitem}
\usepackage{amsmath}
\usepackage{amsthm}
\usepackage{booktabs}
\usepackage{subcaption}
\usepackage{pifont}

\usepackage{multirow}
\usepackage[table]{xcolor}

\theoremstyle{plain}
\usepackage{newfloat}
\usepackage{listings}
\DeclareCaptionStyle{ruled}{labelfont=normalfont,labelsep=colon,strut=off} 
\floatstyle{ruled}
\newfloat{listing}{tb}{lst}{}
\floatname{listing}{Listing}

\newtheorem{definition}{Definition}

\newtheorem{proposition}{Proposition}

\newtheorem{assumption}{Assumption}
\ifdefined\aaaianonymous
    \title{AAAI Press Anonymous Submission\\Instructions for Authors Using \LaTeX{}}
\else
    \title{AAAI Press Formatting Instructions \\for Authors Using \LaTeX{} --- A Guide}
\fi

\title{Rethinking Crystal Symmetry Prediction: \\A Decoupled Perspective}
\author{
    Liheng Yu,\textsuperscript{\rm 1}
    Zhe Zhao,\textsuperscript{\rm 1}
    Xucong Wang,\textsuperscript{\rm 1}
    Di Wu,\textsuperscript{\rm 1}
    Pengkun Wang\textsuperscript{\rm 1,2}\thanks{Corresponding author.} \\ 
}

\affiliations{
    \textsuperscript{\rm 1} University of Science and Technology of China, Hefei 230236, China\\
    \textsuperscript{\rm 2} Suzhou Institute for Advanced Research, University of Science and Technology of China, Suzhou 215123, China\\
    \{yuliheng, zz4543, xuco, wdcxy\}@mail.ustc.edu.cn, pengkun@ustc.edu.cn, 
}

\usepackage{bibentry}

\begin{document}
\maketitle

\begin{abstract}
Efficiently and accurately determining the symmetry is a crucial step in the structural analysis of crystalline materials. Existing methods usually mindlessly apply deep learning models while ignoring the underlying chemical rules. More importantly, experiments show that they face a serious sub-property confusion (\texttt{SPC}) problem. To address the above challenges, from a decoupled perspective, we introduce the XRDecoupler framework, a problem-solving arsenal specifically designed to tackle the \texttt{SPC} problem. Imitating the thinking process of chemists, we innovatively incorporate multidimensional crystal symmetry information as superclass guidance to ensure that the model's prediction process aligns with chemical intuition. We further design a hierarchical PXRD pattern learning model and a multi-objective optimization approach to achieve high-quality representation and balanced optimization. Comprehensive evaluations on three mainstream databases (e.g., CCDC, CoREMOF, and InorganicData) demonstrate that XRDecoupler excels in performance, interpretability, and generalization. 
\end{abstract}

\section{Introduction}
Symmetry determination is a crucial step in powder X-ray diffraction (PXRD) based crystal structure prediction \cite{altomare2004space,spence1992symmetry}. The space group \cite{koster1957space} defines the symmetry characteristics of the crystal, including rotation, reflection, and inversion operations. These symmetries are fundamental to understanding the crystal structure and its properties \cite{o2020crystal,bhagavantam1949crystal}. Once the space group is established, scientists can build and optimize crystal structure models based on that symmetry \cite{han2025efficient,evans2011introduction}. Incorrect selection of space groups can lead to inaccurate or unreasonable structural models. Therefore, \textbf{\textit{efficiently and accurately identifying space group types remains a significant challenge.}}

\begin{figure}[t]
    \centering
    \includegraphics[width= \linewidth]{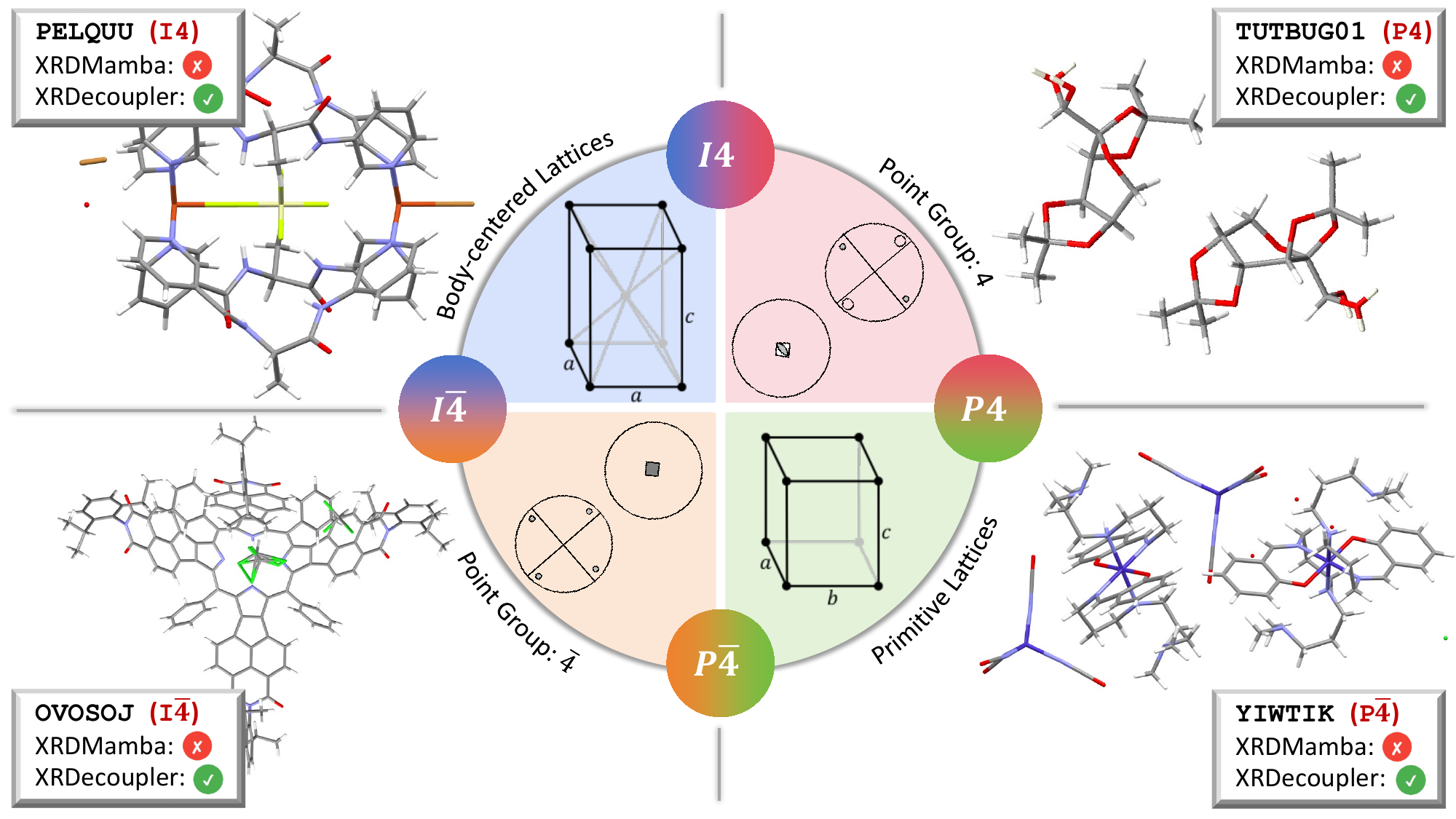}
    \caption{The \texttt{SPC} problem in the space group identification. Different colored blocks represent various symmetry classification systems, such as lattice types and point group types. We illustrate four space groups that current methods often confuse: \texttt{I4}, \texttt{I}$\overline{\texttt{4}}$, \texttt{P4}, and \texttt{P}$\overline{\texttt{4}}$. These space groups are intertwined and may belong to a coarser classification. We also present four representative crystal samples, i.e., \texttt{PELQUU}, \texttt{OVOSOJ}, \texttt{TUTBUG01}, and \texttt{YIWTIK}, demonstrating the capability of our method to decouple these confusions.} 
    \label{fig:confusion}\vspace{-4mm}
\end{figure}

 
As an early typical approach, many researchers attempted to apply a simple neural network for symmetry recognition within specific ranges, advancing the application of machine learning in this task \cite{park2017classification,ziletti2018insightful,oviedo2019fast,vecsei2019neural,dong2021deep}. \textit{However, these methods often remain limited to a single space group or a small subset of space groups, or specific material categories, which restricts their applicability and prevents broader generalization across different materials or space groups.} 
More importantly, simply imposing deep learning models on this task \textit{does not align with chemical intuition}. Although the latest methods attempt to enhance model architectures from a theoretical perspective of pattern \cite{yu2024xrdmamba}, they still face a challenging dilemma: \textit{models tend to mistake two crystalline materials with common sub-properties (such as lattice types) for belonging to the same space group (as shown in Figure~\ref{fig:confusion}), i.e., \textbf{the sub-property confusion (\texttt{SPC}) problem}}.
 
To clarify the essence of \texttt{SPC}, we systematically rethought existing methods and found that the causes of \texttt{SPC} are multifaceted: \ding{182} Previous methods directly process PXRD patterns in a conventional sequential manner, lacking the necessary chemical knowledge to guide the process. \ding{183} The optimization direction of these models often tends to favor latent sub-properties. \ding{184} The information overlap among sub-properties makes it difficult for the model to distinguish between closely related structures. 
To illustrate the \texttt{SPC} problem faced in space group prediction tasks more intuitively, we provide a clear diagram, as shown in Figure~\ref{fig:confusion}. Based on crystallography, the space groups \texttt{I4}, \texttt{I}$\overline{\texttt{4}}$, \texttt{P4}, and \texttt{P}$\overline{\texttt{4}}$ are categorized into two lattice types: I (body-centered) and P (primitive). They are also divided into two point group types: \texttt{4} (fourfold rotation axis) and $\overline{\texttt{4}}$ (fourfold roto-inversion axis). As a result, these four space groups share common sub-properties, making them particularly prone to confusion during classification. We also present the recognition results of different methods on these easily confused samples. \textit{\textbf{The SOTA model}} (XRDMamba \citep{yu2024xrdmamba}) \textit{\textbf{appears to be hindered by the \texttt{SPC}, while our method successfully predicts the space group types for each crystal.}}
 
Our method, \textbf{\textit{XRDecoupler}}, is a problem-solving arsenal specifically designed to tackle the \texttt{SPC} problem in space group prediction. We rethink \texttt{SPC} faced in previous research according to the thinking process of chemists and innovatively incorporate crystal systems, Bravais lattice types, and point groups as superclasses guidance, ensuring that the model's predictions align with chemical rules and enhancing its ability to distinguish between easily confused samples ($\blacktriangleright$ \textbf{solving Cause \ding{182}}). Additionally, different superclasses may focus on different aspects of the PXRD pattern. To address this, we propose a hierarchical PXRD pattern learning model that captures both local and global pattern information inherent in PXRD patterns, enabling efficient multi-superclass learning ($\blacktriangleright$ \textbf{solving Cause \ding{184}}). Furthermore, we utilize a multi-objective optimization approach to ensure that the model does not favor specific superclass tasks at the expense of others, guiding it toward an optimal training path ($\blacktriangleright$ \textbf{solving Cause \ding{183}}). To validate the effectiveness of XRDecoupler, we conducted comprehensive experimental evaluations on the well-known Cambridge Crystallographic Data Centre (CCDC) database \cite{allen1979cambridge}, CoREMOF \cite{chung2019advances}, and InorganicData \cite{salgado2023automated}. The results indicate that XRDecoupler significantly outperforms other state-of-the-art baselines. 

Our contributions are summarized as follows:
\begin{itemize}[left=0.5em]
    \item[\ding{182}] \textit{\textbf{Crucial Problem and Fresh Perspective}}: For the first time, we rethink existing methods and analyze the \texttt{SPC} problem faced in space group prediction tasks. From a decoupled perspective, we introduce the XRDecoupler framework to decouple it.
    \item[\ding{183}] \textit{\textbf{Novel Mechanism}}: Imitating the thinking process of chemists, we innovatively incorporate multidimensional crystal symmetry information as superclasses guidance, ensuring predictions align with chemical intuition.
    \item[\ding{184}] \textit{\textbf{Reasonable Pattern Learner}}: We proposed a hierarchical PXRD Pattern learning model that explicitly models the local and global pattern information in PXRD.
    \item[\ding{185}] \textit{\textbf{Brilliant Performance}}: Evaluations on well-known databases demonstrated the superior performance, interpretability, and generalization of XRDecoupler.
\end{itemize}

\section{Related Works}\label{asec:rw}
\subsubsection*{$\blacktriangleright$ Crystalline Space Group Prediction.}
Identifying crystal space groups is vital for structure prediction. While early work relied on computational methods \citep{werner1985treor}, recent efforts have focused on deep learning. Initial convolutional neural networks showed promise but were often trained or tested on limited datasets \citep{park2017classification, ziletti2018insightful, oviedo2019fast, vecsei2019neural, dong2021deep}. Subsequent models like NPCNN \citep{salgado2023automated} utilized more comprehensive data but struggled with accuracy. In response, RCNet \citep{chen2024crystal} improved performance by customizing categories and using residual structures, while XRDMamba \citep{yu2024xrdmamba} pioneered the integration of chemical knowledge and the Mamba architecture \citep{gu2023mamba,gu2023modeling}. However, these methods often neglect relevant chemical principles and suffer from high confusion between similar space groups. Our approach addresses this by designing a hierarchical framework and introducing multidimensional superclass knowledge to guide model optimization, thereby mitigating confusion and enhancing generalization.

\subsubsection*{$\blacktriangleright$ Superclass Learning.}
Superclass learning improves model performance by incorporating high-level class groupings as intermediate supervision. This technique has proven effective across diverse domains, such as improving feature differentiation in image classification \citep{dehkordi2022multi, wang2022sar}, mitigating data imbalance \citep{zhou2018deep}, and learning high-level relationships in graph neural networks \citep{du2023superdisco}. The benefits of this approach are well-documented. Given the inherent hierarchical classification of space groups, we introduce superclass learning to this task for the first time. This guides our model to learn more detailed structural knowledge about crystals, significantly enhancing its recognition performance. A more comprehensive review is provided in \textbf{\underline{Supplement H}}.

\section{Motivation}
\subsubsection*{$\blacktriangleright$ Space Group Prediction Task.} The PXRD data can be viewed as two vectors, $A$ and $I$, both of length $L$, representing the diffraction angles and diffraction intensities on the PXRD pattern, respectively. In crystallography, the space group is a description of the symmetry of crystals, with a total of 230 theoretically existing space groups. We define $D = \{X_i, Y_i\}_{i \in [n]}$ as a crystal dataset, where $ X_i = (A_i, I_i) $ represents the PXRD pattern data of a crystal, and $ Y_i \in \{0,1,...,229\} $ denotes the space group label of that crystal. We represent the process of space group prediction as a mapping relationship: $ f(X_i) \rightarrow Y_i $. We provide definitions about space groups in \textbf{\underline{Supplement A.1}} for understanding.
  
\subsubsection*{$\blacktriangleright$ Multidimensional Symmetry as Superclasses.} The superclass refers to a more general class in a hierarchical classification structure that contains other classes (called subclasses). For crystal structures, aside from space groups, there are many coarse-grained partition rules \citep{nespolo2018crystallographic}. For example, as shown in \textbf{\underline{Figure 1}} in \textbf{\underline{{Supplement A}}}, based on the symmetry of geometric morphology, crystals can be divided into 7 crystal systems \citep{o2020crystal}; based on primitive point symmetry, they can be classified into 32 point groups \citep{bradley2009mathematical}; according to the point symmetry of Bravais lattices, there are 7 lattice systems; based on the spatial symmetry of Bravais lattices, there are 14 types of Bravais lattices \citep{pitteri1996definition}; and based on combinations of point symmetry and translational symmetry, they can be classified into 73 algebraic crystal classes \cite{wilson1990space}, among others \cite{hahn1983international, prince2004international}. \textit{The partition rules mentioned above can all be viewed as superclasses of space groups, describing only the symmetric properties of a certain part of the crystal.} We also provide an explanation of the superclass mechanism in \textbf{\underline{{Supplement A.3}}}.

\section{Rethinking `Culprits' Behind \texttt{SPC}}
\subsubsection*{$\blacktriangleright$ Culprit 1: Lack of Chemical Knowledge.}
Existing models are typically adapted directly from other deep learning tasks, predominantly utilizing rudimentary convolutional neural networks. Although these models have yielded initial results, they often fall short in enhancing performance further. \textit{The current model setup fails to consider chemical knowledge integral to material structure analysis, making it challenging for the model to genuinely comprehend the intricate details of the structural patterns within the data}. Confusion often arises when the model only captures the coarse-grained aspects of the data. Hence, a model tailored with integrated chemical knowledge becomes indispensable to enhance the efficacy of space group identification.

\subsubsection*{$\blacktriangleright$ Culprit 2: Bias in Optimization.}
Successful prediction of a space group requires the model to accurately identify the sample's sub-symmetry properties, such as lattice type, point symmetry, and translation symmetry. During optimization, supervisory signals for each property must be transmitted to the model's parameters to guide learning. \textit{When the unique space group label is the sole supervisory signal, it can cause insufficient signals and imbalanced gradients across sub-properties, leading the model to learn only some of them effectively}. Consequently, the model struggles to distinguish between samples that share these learned sub-properties but belong to different space groups, causing significant confusion. For example, as shown in Figure \ref{fig:wrong_sample_acc}, we analyzed samples misclassified by a state-of-the-art method and evaluated the performance for each sub-attribute. It is evident the model learned more about lattice types and crystal systems but performed poorly on point group characteristics. As a result, \textit{the model exhibits significant confusion with samples that share the same lattice type and crystal system but differ in their point groups}.

\subsubsection*{$\blacktriangleright$ Culprit 3: Information Overlap Between Labels.}
The space group category serves as a fine-grained structural classification criterion, encompassing numerous sub-properties of crystal structures. For two distinct categories, they often share certain structural sub-properties, a phenomenon we refer to as information overlaps between labels. In the following, we will analyze how this information overlap can lead to confusion from the perspective of information theory \cite{ash2012information,batina2011mutual}.

For a given sample $X$, we denote its true confidence label as $Y_{truth}$ and any other space group class label as $Y_{other}$. In this context, for sample $X$, label $Y_{other}$ represents an incorrect label. 
\begin{assumption}
For the space group class label $Y$, it can be uniquely represented by $k$ independent structure sub-properties, i.e., $Y = (y_{1},y_{2},...,y_{i},...,y_{k})$, where $y_i$ denotes $i$-th sub-properties.
\label{assump:sub_prop}
\end{assumption}

\begin{figure}[t] 
    \vspace{-0.2cm}
    \centering
    \includegraphics[width=0.9\linewidth]{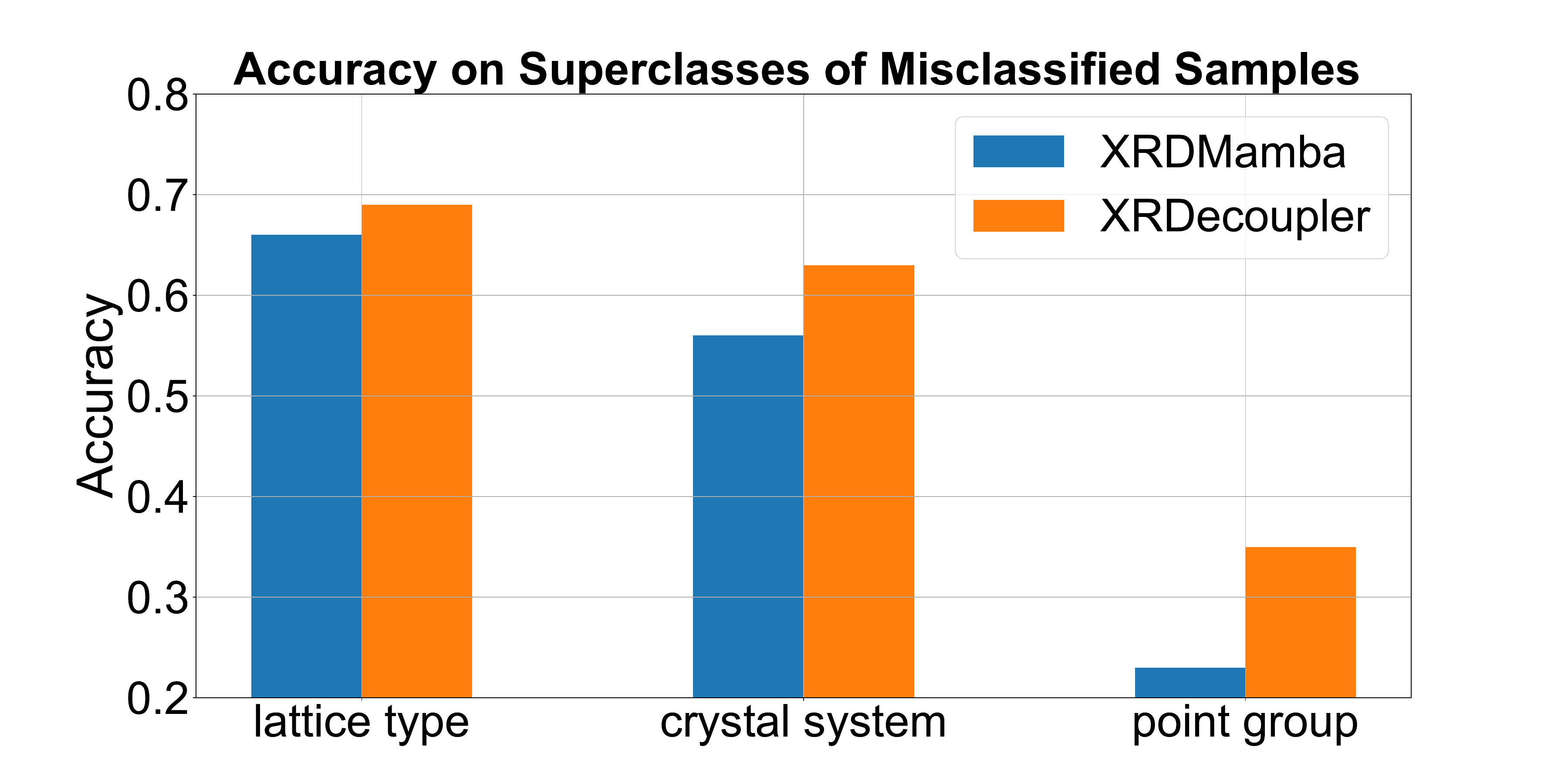}
    \caption{Accuracy statistics of samples with misidentified space groups on three sub-properties (e.g., lattice type, crystal system, and point group) on the SOTA method (XRDMamba \citep{yu2024xrdmamba}) and our proposed XRDecoupler.} \vspace{-0.2cm}
    \label{fig:wrong_sample_acc}
\end{figure} 

\begin{figure*}[t!]
    \centering
    \includegraphics[width= 0.85\textwidth]{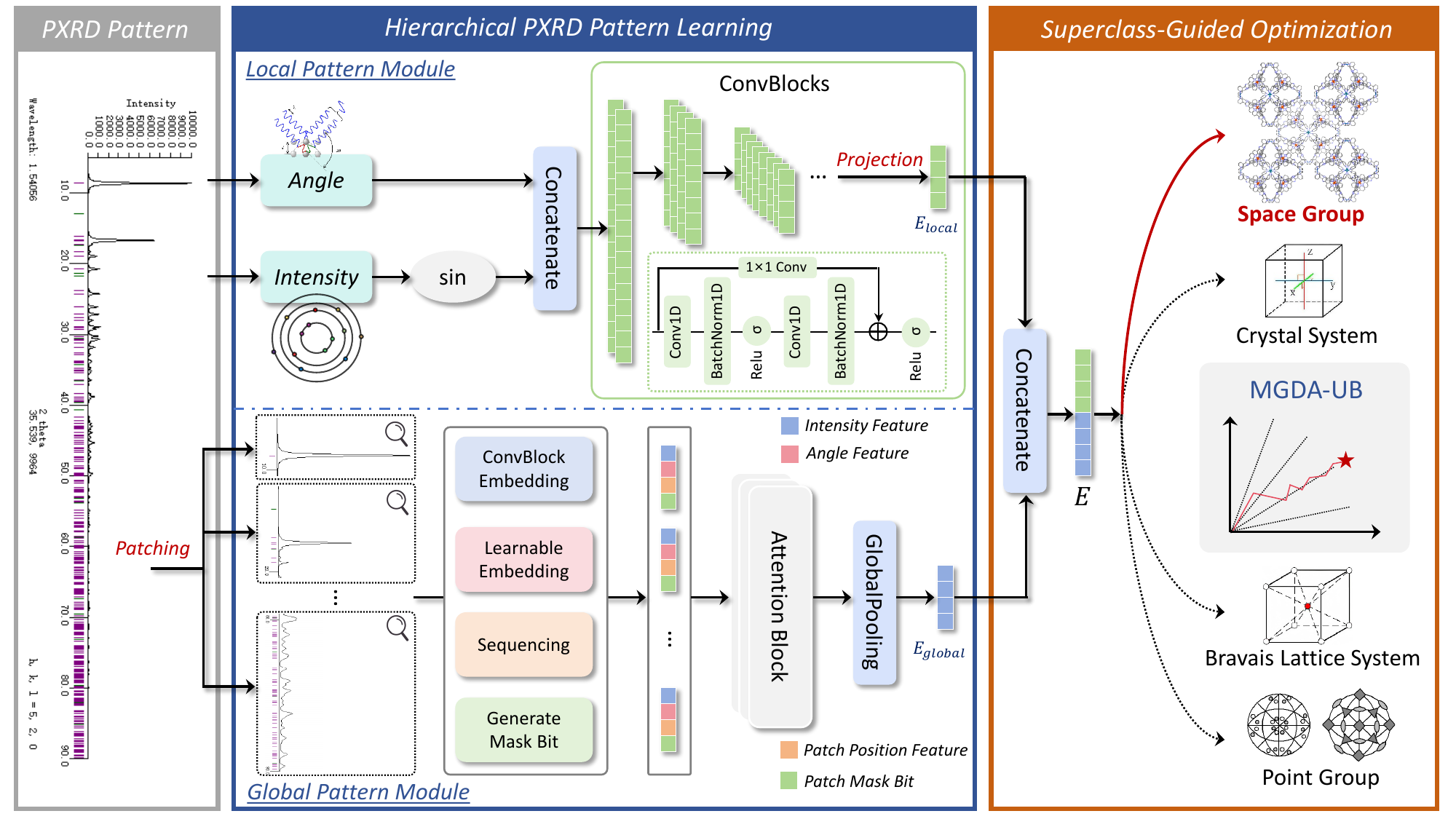}
    \caption{Overview of XRDecoupler.}
    \label{fig:model}
\end{figure*}

\begin{proposition}
For $Y_{truth}= ({y}^{truth}_{1},{y}^{truth}_{2},...,{y}^{truth}_{k})$ and $Y_{other}= ({y}^{other}_{1},{y}^{other}_{2},...,{y}^{other}_{k})$, there are some same structure sub-properties. That is, there exists $s$,$t$, such that
\begin{equation}
\begin{aligned}
    &\begin{cases}
    y_{{p_d}}^{truth} = y_{{p_d}}^{other} , d=[1,2,...,s]\\
    y_{{q_d}}^{truth} \neq y_{{q_d}}^{other} , d=[1,2,...,t]
    \end{cases} 
    \\ &\text{s.t.} \quad s + t = k, \;\; \{[p_i],[q_i]\} = \{1,2,3,...,k\}    
\end{aligned}
\end{equation}
\end{proposition}
For simplicity, let $M = \{y_{{p_d}}^{truth}\} ,d\in[1,2,...,s]$ represent the overlap information between $Y_{truth}$ and $Y_{other}$. Let $y_{truth} = \{y_{{q_d}}^{truth}\},d\in[1,2,...,t]$ and $y_{other} = \{y_{{q_d}}^{other}\},d\in[1,2,...,t]$ denote the non-overlapping information between the two labels. Then, we have $Y_{truth} = (M,y_{{truth}})$ and $Y_{other} = (M,y_{{other}})$. Usually, our aim is to have the model maximize the mutual information ${I(E;Y_{truth})}$ between sample representation and truth label during training, where $E$ denotes the representation obtained from the encoder. 

\begin{proposition}
    If there is significant overlap information between labels, maximizing the mutual information between sample representation and truth label will results in the mutual information of sample representation and overlap label being maximized as well, specifically:  
    \label{prop:ol_inf}
    \begin{equation}    
        \begin{aligned}
            I(E;Y_{truth}) &  \uparrow\;\; \Rightarrow \;I(E;Y_{other})\uparrow \\
            & \text{s.t.} \;\; I(M;Y_{truth})\gg  {I(y_{truth};Y_{truth})}.
        \end{aligned}
    \end{equation}
\end{proposition}
We provide a theoretical analysis in \textbf{\underline{Supplement B.1}}. Proposition \ref{prop:ol_inf} shows, when there is significant overlap in information between labels, simply maximizing the mutual information between the representation and the truth label is insufficient. Thus, we observe the difference between the mutual information, which serves as a reliable indicator of the model's confusion in executing a classification. 
\begin{definition}[Difference $Diff$ in mutual information]
    The difference $Diff$ in mutual information between two labels is equal to the mutual information of their non-overlapping parts.
    \begin{equation}
        \begin{aligned} 
        Diff
        &= I(E;Y_{truth})-I(E;Y_{other})\\
        &= I(E;y_{truth})-I(E;y_{other})
        \end{aligned}
    \label{eqa:Ixy_diff}
    \end{equation}
\end{definition}

We provide the corresponding proof in \textbf{\underline{Supplement B.2}}.

\section{XRDecoupler is A `Nemesis' of \texttt{SPC}}
Here, we propose a novel method for determining space groups from PXRD, called XRDecoupler.
This method includes a superclass-guided optimization framework ({\textbf{for Culprit 2 \& 3}}) and a hierarchical PXRD pattern learning model (\textbf{for Culprit 1}), effectively addressing the significant \texttt{SPC} issues associated with previous methods. 
An overview of XRDecoupler is illustrated in Figure \ref{fig:model}. 

\subsection{Superclass-Guided Optimization ($\blacktriangleright$ Culprit 2 \& 3)}

\subsubsection*{$\blacktriangleright$ Mitigating Information Overlap.} Considering significant information overlap between space group labels, we expect that the model focuses more on label-specific features and devotes less attention to overlapping information during the classification. Thus, we also aim to maximize $Diff$. 
\begin{proposition}
    The process of maximizing $Diff$ can be interpreted as maximizing the mutual information between the sample and each structured sub-property.
    \begin{equation}
        \begin{aligned}
            Diff &= I(E;y_{truth})-I(E;y_{other}) \\
            &= \sum_{i=1}^t{[I(E;y_{i}^{truth})-I(E;y_{i}^{other})]} \\ 
        \end{aligned}
    \end{equation}
\end{proposition}
We provide the corresponding proof in \textbf{\underline{Supplement B.3}}. From this proposition, we can optimize the model's ability to learn each sub-property to alleviate the confusion phenomenon. However, for the moment, sub-properties are abstractions. Let's discuss how to optimize them.

\subsubsection*{$\blacktriangleright$ Superclass Guidance.}
Superclasses of space groups describe a broad property of crystal structures, which aligns with the idea of structured sub-property mentioned earlier.
Therefore, it is logical to introduce various superclasses of space groups to represent the sub-properties above, and optimize the sub-properties in the same way as optimizing the superclasses. Through the oversight of these superclasses, the model's focus on overlapping information can be minimized, thereby enhancing the efficacy of space group identification. Specifically, we focus on the $T$ types of superclasses of space groups. We use $C$ to represent the space group category, $C^{Sup}_i$ to denote the $i$-th type of superclass category of space groups, $y_i$ to indicate the space group category corresponding to the $i$-th sample, and $y^{Sup_j}_i$ to represent the category of the $j$-th superclass corresponding to the $i$-th sample. We define $Y_i = \{y_i\} \cup \{y^{Sup_j}_i\}_{j \in [T]}$ to represent the set of all categories to which the $i$-th sample belongs. Then, we assume that the classifier for space group classification is $Cls$, and the classifiers for the various superclasses are denoted as $\{ \text{Cls}^{Sup_t} \}_{t \in [T]}$. Our optimization objective is given by:    
\begin{equation} \small
\begin{aligned}
    &\min {-I(E;y)-\beta\sum_{t\in[T]}{I(E;y^{Sup_t})}} \\
    &\Leftrightarrow \min{\mathbb{E}_{(X,Y)}{ -\log p(y|E)-\sum_{t\in [T]}\log p(y^{Sup_t}|E) } }\\
    &\Leftrightarrow \min {L_{sp}(\text{Cls}(f(X)), y) + \beta \sum_{t \in [T]} L^{t}(\text{Cls}^{Sup_t}(f(X)), y^{Sup_t})},
\end{aligned}
\end{equation}
where $L_{sp}(\cdot)$ and $L^t(\cdot)$ are the cross-entropy loss, $f$ is an encoder satisfying $E = f(X)$, and $\beta$ is a hyperparameter. 

\subsubsection*{$\blacktriangleright$ Example.} In the confusion cases illustrated in Figure \ref{fig:confusion}, we clarify the representations of samples from the \texttt{I4} \& \texttt{P4} categories, as well as \texttt{I}$\overline{\texttt{4}}$ \& \texttt{P}$\overline{\texttt{4}}$, by introducing the Bravais lattice type superclass as supervision. Additionally, we use the point group type superclass to differentiate between \texttt{I4} \& \texttt{I}$\overline{\texttt{4}}$, and \texttt{P4} \& \texttt{P}$\overline{\texttt{4}}$ samples, thereby enhancing the discriminative capability of these four space groups.

\subsubsection*{$\blacktriangleright$ Optimization Process.} Following the introduction of superclasses, each one provides a multitude of detailed supervision signals for the model's optimization process. Nevertheless, as illustrated in Figure \ref{fig:Optimization}(left), Culprit 2 remains unresolved, with a persistent skew in the optimization process. In this process, the model's loss on certain superclasses initially decreases, only to be followed by the gradual optimization of the remaining superclasses once these have been refined. This skewed optimization process is detrimental to enhancing the model's generalization performance and also impacts the benefits derived from the superclasses. Therefore, we have transformed the optimization process into a multi-objective optimization process. We consider the space group and the $T$ superclasses as $T+1$ independent classification objectives, thereby optimizing the model's learning as a multi-objective optimization, i.e., 
\begin{equation}\small
    \begin{aligned}
    & \min_{W}\hat{L}  (W^{Enc},W^{Sup}_1,...,W^{Sup}_{T+1}) \\
    & =\min_W(L^1(W^{Enc},W^{Sup}_1),...,L^{T+1}(W^{Enc},W^{Sup}_{T+1}))
    \end{aligned}
\end{equation}
where $W^{Enc}$ represents the parameters of the encoder, $W^{Sup}_t$ represents the parameters of the classifier for each class and $W = \{W^{Enc} \} \bigcup \{W^{Sup}_t\}_{t\in [T+1]}$ represents all parameters.
The goal of multi-objective optimization is achieving Pareto optimality.
\begin{definition}[Solution $W$]
A solution $W$ dominates a solution $\overline{W}$ if $\hat{L}(W^{Enc},W^{Sup}_t) \leq \hat{L}({\overline{W}}^{Enc},\overline{W}^{Sup}_t)$ for all objectives $t \in [T+1]$ and $\hat{L}(W^{Enc},W^{Sup}_1,W^{Sup}_2,...,W^{Sup}_{T+1}) \neq \hat{L}(\overline{W}^{Enc},\overline{W}^{Sup}_1,\overline{W}^{Sup}_2,...,\overline{W}^{Sup}_{T+1}) $.
\end{definition}
\begin{definition}[Solution $W^*$]
A solution $W^*$ is Pareto optimal if there exists no solution $W$ that dominates $W^*$.
\end{definition} 
Therefore, the optimization process of space group and superclass is transformed into finding a Pareto optimal solution $W^*=(W^{Enc^*},{W_1^{{Sup}^*}},...,{W_t^{{Sup}^*}})$. Inspired by MGDA-UB\citep{sener2018multi}, we further transform the optimization problem into solving a set of weights $\{\alpha_i\}$:
\begin{equation}\small
    \begin{aligned}
        \min_{\alpha_1,\alpha_2,...,\alpha_{T+1}} ||\sum^{T+1}_{t=1}\alpha_t\nabla_{W^{Enc}}L^t(W^{Enc},W^{Sup}_t)||^2_2, \\
        \text{s.t.}  \quad \sum^{T+1}_{t=1}\alpha_t=1,\alpha_t>=0 \;\;\forall t.
    \end{aligned}
\end{equation}
where $\alpha_t$ denotes the weight of the $t$-th objective.

Then, we use the FRANK-WOLFESOLVER algorithm (FW) to obtain a gradient direction that improves all $T+1$ classification tasks. The optimization process is as follows.
\begin{equation}\small
    \begin{aligned}
        (1) & W^{Sup}_t = W^{Sup}_t - \eta  \nabla_{W^{Sup}_t} L(W^{Enc},W^{Sup}_t) \forall t \in [T+1]\\ 
        (2) & \alpha_1,...,\alpha_t = FW(W^{Enc},{\{W^{Sup}_t\}}_{t\in [T+1]})\\
        (3) & W^{Enc} = W^{Enc}-\eta \sum^{T+1}_{t=1}\alpha_t\nabla_{W^{Enc}}L^t(W^{Enc},W^{Sup}_t)\\
    \end{aligned}
\end{equation}
where $\eta$ represents the learning rate.

\subsection{Hierarchical PXRD Pattern Learning ($\blacktriangleright$ Culprit 1)}
The introduction of superclasses brings a wealth of structural knowledge to the model, including point symmetry, crystal structure periodicity, lattice structures, and other rich information. This encompasses global and local information about the crystal structure, raising the bar for the encoder to learn more refined knowledge. As Culprit 3 mentioned, the original models struggle to capture such fine-grained knowledge and are unable to learn detailed representations that incorporate these specific insights. 

Therefore, we propose a new hierarchical Pattern learning model tailored to the characteristics of space groups. This model consists of two main components: (i) Local Pattern module (LP), which captures local information between adjacent peaks in the PXRD pattern, outputting a $k_{local}$-dimensional representation $E_{local}$; (ii) Global Pattern module (GP), which captures global information from the PXRD pattern, outputting a $k_{global}$-dimensional representation $E_{global}$. We combine these two representations to obtain the final output representation of the model:
\begin{equation}
    E = \text{Concat}(E_{global}, E_{local})
\end{equation}
This representation is used for the subsequent superclass-guided optimization process.

\makeatletter
\newcommand{\tabcaption}[1]{\def\@captype{table}\caption{#1}}
\makeatother
\begin{figure*}[t!]
\centering
    \begin{minipage}{0.62\textwidth}
    \centering
    \scalebox{0.70}{
        \begin{tabular}{l  c c c c   c c c }
            \toprule[2pt]
            \multicolumn{1}{l}{\multirow{2}{*}{Method}} & \multicolumn{3}{c}{Accuracy (\%) on MOF} & & \multicolumn{1}{c}{\multirow{2}{*}{F1 Score (\%)}} & & \multicolumn{1}{c}{\multirow{2}{*}{Recall (\%)}} \\
            \cmidrule{2-4}
            & Top-1 & Top-2 & Top-5 & & & &\\
            \midrule
            MLP        & 9.10  {($-$29.90)}  & 15.10 {($-$41.30)}  & 30.10 {($-$47.48)}  && 6.43  {($-$16.17)} && 5.24  {($-$16.26)} \\
            CNN        & 39.00 {($+$0.00)}  & 56.40 {($+$0.00)}  & 77.58 {($+$0.00)}  && 22.60 {($+$0.00)} && 21.50 {($+$0.00)} \\
            \midrule
            NoPoolCNN  & 38.20 {($-$0.80)}   & 51.80 {($-$4.60)}   & 71.12 {($-$6.46)}   && 34.47 {($+$11.87)} && 31.84 {($+$10.34)} \\
            RCNet           & 59.00 {($+$20.00)} & 73.70 {($+$17.30)} &   88.37 {($+$10.79)}    &&   41.29 {($+$18.69)}    && 40.38 {($+$18.80)}  \\
            XRDMamba         & \underline{72.20} {($+$33.20)} & \underline{85.20} {($+$28.80)} & \underline{93.42} {($+$15.84)} && \underline{47.59} {($+$24.99)} && \underline{46.00} {($+$24.50)} \\
            \midrule
            \rowcolor{gray!10} XRDecoupler    & \textbf{80.09} {($+$41.09)} & \textbf{90.11} {($+$33.71)} & \textbf{96.26} {($+$18.68)} && \textbf{56.72} {($+$34.12)} && \textbf{55.18} {($+$33.68)} \\
            \bottomrule[2pt]
        \end{tabular}
    }
    \end{minipage}
    \hspace{0.0\textwidth}
    \begin{minipage}{0.33\textwidth}
    \centering
    \scalebox{0.70}{
        \begin{tabular}{|c c c }
            \toprule[2pt]
            \multicolumn{3}{|c}{Accuracy (\%) on MOF-Balanced} \\
            \cmidrule{1-3}
             Top-1 & Top-2 & Top-5\\
            \midrule
             4.10  {($-$18.8)}  & 5.40 {($-$27.00)}  & 8.21 {($-$38.79)}  \\
             22.90 {($+$0.00)}  & 32.40 {($+$0.00)}  & 47.00 {($+$0.00)}  \\
            \midrule
             33.80 {($+$10.90)}   & 40.70 {($+$8.30)}   & 50.97 {($+$3.97)} \\
             44.50 {($+$21.60)} & 55.50 {($+$23.10)} &   69.12 {($+$22.12)}    \\
             \underline{48.70} {($+$25.80)} & \underline{61.70} {($+$29.3)} & \underline{74.83} {($+$27.83)}  \\
            \midrule
            \rowcolor{gray!10}  \textbf{58.87} {($+$35.97)} & \textbf{72.42} {($+$40.02)} & \textbf{85.22} {($+$38.22)}  \\
            \bottomrule[2pt]
        \end{tabular}
    }
    \end{minipage}
    \tabcaption{Evaluation on MOF subset (left) and MOF-Balanced subset (right) of CCDC dataset with SOTA methods. \textbf{Bold} indicates the best performance while \underline{underline} indicates the second best. {\textbf{($+$)}} and {\textbf{($-$)}} indicate the the relative gain with CNN. }
    \label{tbl:main_MOF}
\end{figure*}

\subsubsection*{$\blacktriangleright$ Local Pattern Module.}
The input vectors of the local pattern module are $A$ and $I$, which represent the diffraction angles and intensities on the PXRD pattern, respectively. 
According to Bragg's law\citep{pope1997x}, $2d \sin(\theta) = n \lambda$, the interplanar spacing is inversely related to $\sin(\theta)$. Thus, we replace $A$ with $\sin(A)$ and concatenate it with the peak intensity $I$ to form an input sequence with $\text{channel} = 2$.

Next, we employ several residual 1D convolution blocks (a kernel size of 3 and a stride of 1) to capture the correlations between adjacent peaks in the PXRD pattern. After the convolution process, we flatten the features and project them to obtain the local pattern representation $E_{local}$. Therefore, the representation process of the local pattern module can be formalized as:
\begin{equation}\small
    E_{local} = \text{Projection}\left(\text{Flatten}\left(\text{ConvBlocks}(\sin(A) \, \oplus \, I)\right)\right)
\end{equation}
where \text{Projection} refers to a linear projection, \text{Flatten} denotes the flattening operation, and \text{ConvBlock} is a submodule composed of convolutional layers, while $\oplus$ represents the concatenation operation. 

\subsubsection*{$\blacktriangleright$ Global Pattern Module.}
The global pattern module first segments the peak intensity data $I$ into $L_p$ consecutive patches of length $p$. Each patch $I_p$ is then processed by convolution blocks to capture local peak correlations, yielding a peak intensity representation, $e_{intensity} \in \mathbb{R}^{k_{intensity}}$. Next, we assign a learnable feature $e_{learnable} \in \mathbb{R}^{k_{learnable}}$ for each patch, allowing the model to adaptively learn representations that characterize the patch. After setting the positional encoding \( e_{position} \) for each patch, we concatenate the peak intensity features and the learnable features, and then add \( e_{position} \) to form the final representation of the patch. Subsequently, we use several attention modules to learn these features, enabling the model to capture the correlations between any two patches and providing more global information. After that, we apply a global pooling layer to obtain the final global representation $E_{global}$.
 
However, we experimentally found that some patches obtained from the PXRD pattern have peak intensities that are entirely zero. These patches hold no value for the model during the learning process and may even have a negative effect on other patches. Therefore, during the calculation of attention, we apply a masking process to these patches, defined as:
\begin{equation}
    e_{mask} = (\max(I_p) == 0).
\end{equation}
This ensures that patches with all zero intensities do not contribute to the attention calculations.  
 
Therefore, the representation process of the global pattern module can be formalized as:
\begin{equation}
\begin{aligned}
    E_{global} &=  \text{GlobalPooling}(\text{AttentionBlocks}(\\
    &(e_{intensity} \oplus e_{learnable}) + e_{postion}, e_{mask}))
\end{aligned} 
\end{equation}
Here, \( e_{intensity} \) represents the peak intensity features of the patch, \( e_{learnable} \) denotes the learnable features for the patch, \( e_{position} \) represents the positional features of the patch, and \( e_{mask} \) indicates whether each patch is masked.

\section{Experiments}
For more detailed information, please refer to \textbf{supplementary material in the extended version} of our paper.

\subsubsection*{$\blacktriangleright$ Dataset and Baselines.} We use the MOF dataset as our main dataset as \citep{yu2024xrdmamba}, which consists of over 280,000 metal-organic frameworks (MOFs) \citep{furukawa2013chemistry,james2003metal,zhou2012introduction} from the Cambridge Crystallographic Data Centre (CCDC) \citep{allen1979cambridge} for our experiments. Furthermore, we were successful in acquiring two additional datasets, CoREMOF \citep{chung2019advances} and InorganicData \citep{salgado2023automated}, to assist in verifying the effectiveness of our method. We provide the details of datasets and processing in \textbf{\underline{Supplement C}}. To ensure fairness in the experiments, we selected several SOTA space group prediction models as our baselines, including {MLP}~\citep{salgado2023automated}, {CNN}~\citep{salgado2023automated}, {NoPoolCNN}~\citep{salgado2023automated}, {RCNet}~\citep{chen2024crystal}, and {XRDMamba}~\citep{yu2024xrdmamba}. Detailed descriptions are provided in \textbf{\underline{Supplement D}}.
 

\subsubsection*{$\blacktriangleright$ Benchmark Results.} 
As shown in Table \ref{tbl:main_MOF}, we outperform all baselines on both the MOF (left) and MOF-Balanced (right) test sets, surpassing the previous SOTA (XRDMamba). These results demonstrate our method effectively alleviates \texttt{SPC} and establishes a new SOTA. More results on CoREMOF and InorganicData are provided in \textbf{\underline{Supplement F.1}}.

\begin{figure*}[t!]
    \centering
    \begin{subfigure}{0.35\linewidth}
        \centering
        \includegraphics[width=\linewidth]{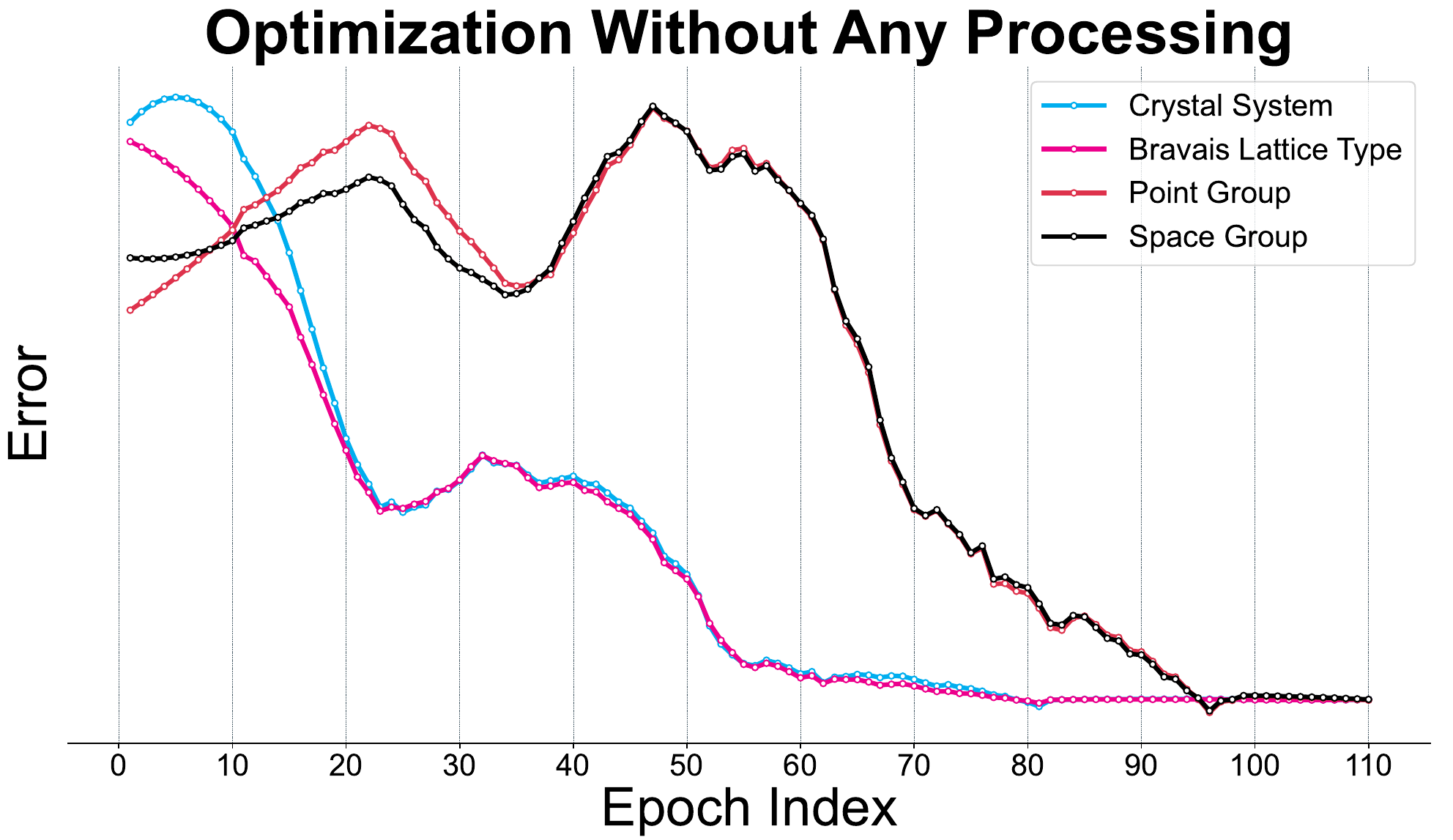} 
    \end{subfigure}
    \begin{subfigure}{0.35\linewidth}
        \centering
        \includegraphics[width=\linewidth]{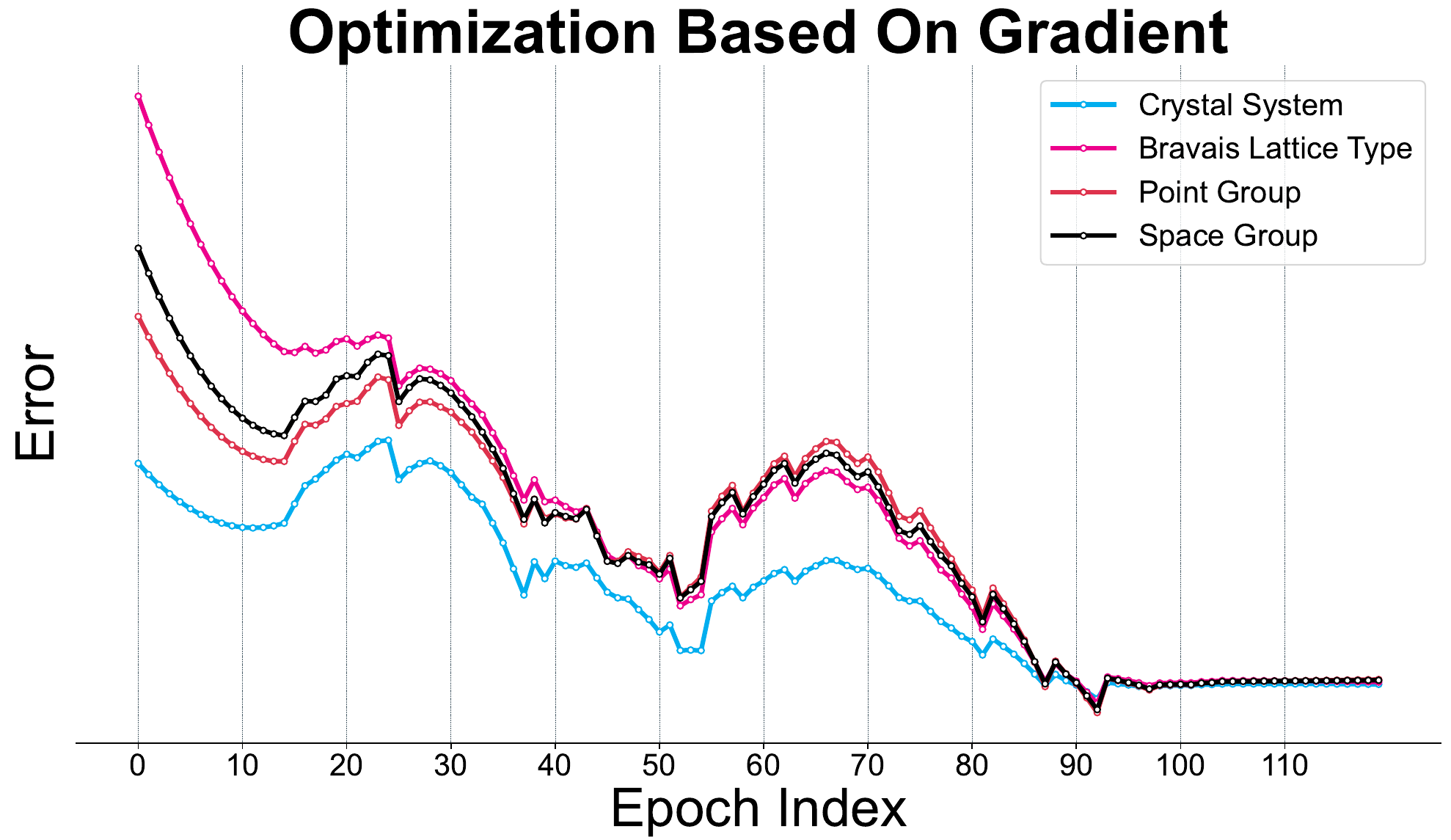} 
    \end{subfigure}
    \hspace{0.02\textwidth}
    \begin{subfigure}{0.25\linewidth}
        \centering
        \includegraphics[width=\linewidth]{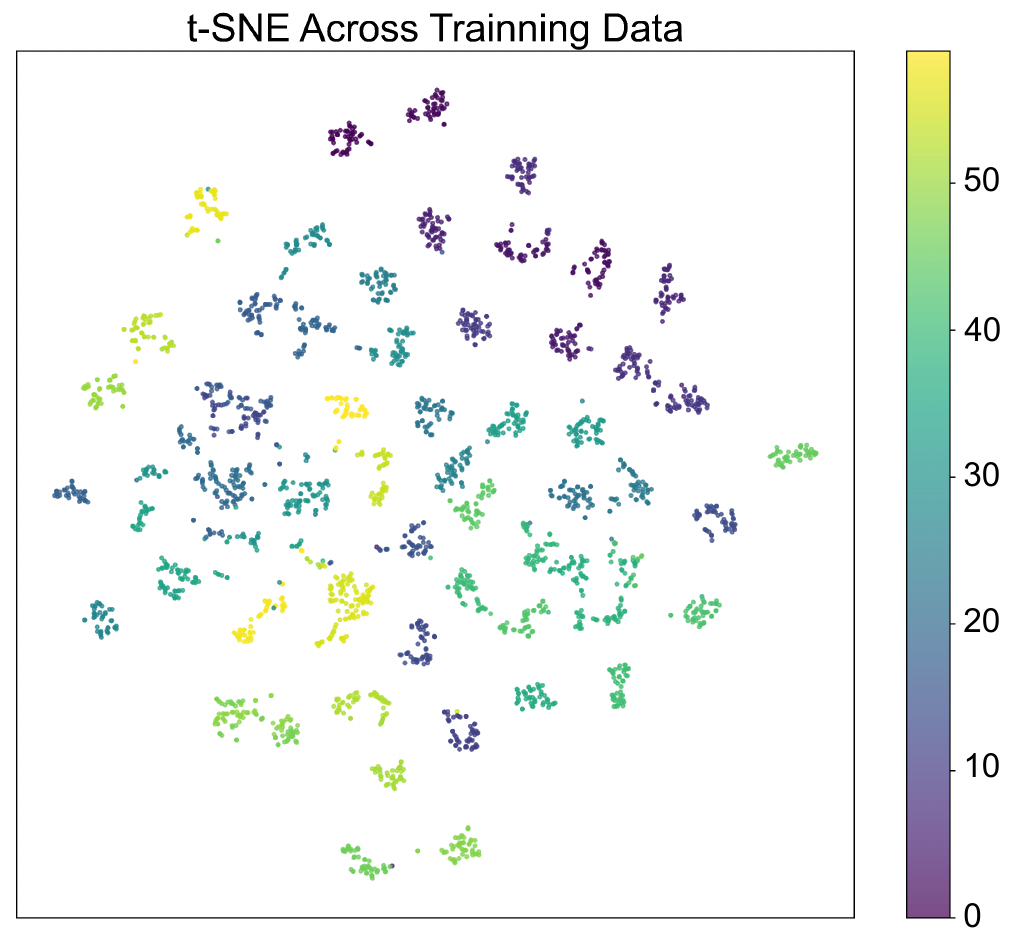}
    \end{subfigure}
    \caption{Trend of the model's training loss. (left) The conventional optimization process of the model on the space group and superclass. (middle) The optimization process of the model on the space group and superclass after introducing the gradient-based optimization method. (right) T-SNE Visualization Analysis of XRDecoupler in the training set.}\vspace{-0.3cm}
    \label{fig:Optimization}
\end{figure*}

\subsubsection*{$\blacktriangleright$ Effectiveness of Gradient-based Optimization.}
Figure \ref{fig:Optimization}(a) shows the loss descent curves for space group and superclasses when the gradient-based multi-objective optimization method is not employed. We can observe that the model initially optimizes the coarse-grained Bravais lattice types and crystal system types. Only after the losses for these superclasses drop to a low level does the model begin to optimize the space groups and point groups. 
Figure \ref{fig:Optimization}(b) presents the loss descent curves for each superclass and space group after introducing the gradient-based multi-objective optimization method. Here, we can see that the optimization directions for each class become consistent, allowing the model to simultaneously learn structural knowledge across multiple dimensions. 

\begin{figure}[t]
    \centering
    \includegraphics[width=0.9\linewidth]{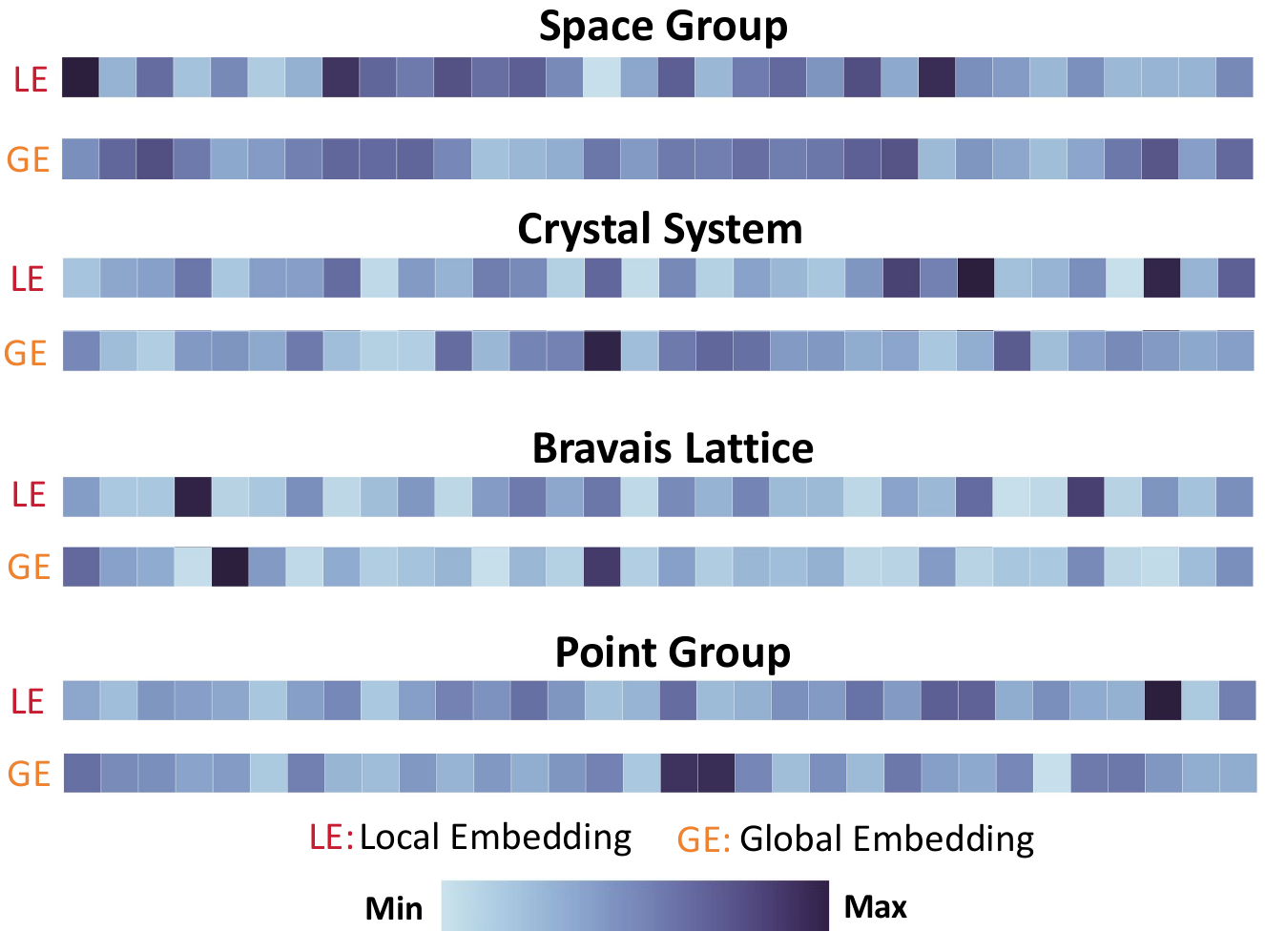}
    \caption{Visualization of the impact of global and local representations on the decision of each superclass.}\vspace{-0.3cm}
    \label{fig:heatmap_feature}
\end{figure}

\subsubsection*{$\blacktriangleright$ T-SNE Visualization Analysis of Confusion Phenomena.}
Figure \ref{fig:Optimization}(right) displays the t-SNE plot of over 3,000 samples randomly sampled from 60 classes in the training set. The figure clearly shows a distinct clustering of features, with noticeable spacing between samples from different space group categories, indicating the superior performance of our model on the training set. We also provide more t-SNE visualization analysis in \textbf{\underline{Supplement F.2}}.

\subsubsection*{$\blacktriangleright$ Dependencies between Representations and Superclasses.}  
To validate the roles of global and local representations, we explored their interdependencies with superclass tasks during classification. We analyzed the classifier weights for space groups and each superclass, observing the importance of different representation locations for the classifier's decision-making by examining the distribution of these weights. The relevant visualizations are shown in Figure \ref{fig:heatmap_feature}. Therefore, we can conclude that in XRDecoupler, both the global and local modules play significant roles and exhibit notable interdependencies across superclass tasks.

\begin{figure}[t]
    \centering
    \includegraphics[width=0.9\linewidth]{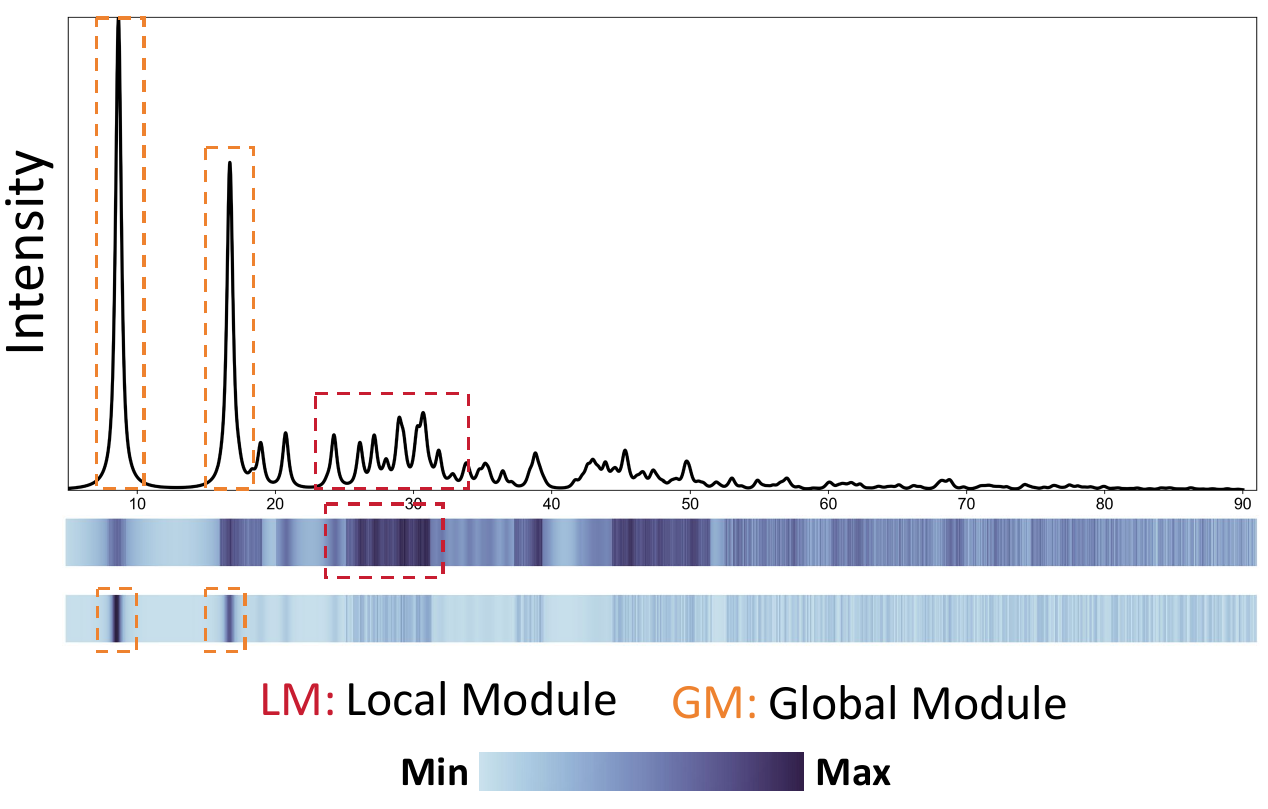}
    \caption{Visualization of the attention of the global and local pattern module to each region in the input PXRD pattern.}
    \label{fig:heatmap_pxrd}
\end{figure}

\subsubsection*{$\blacktriangleright$ Focus Tendency of Global \& Local Pattern Modules.}
To explore the roles of the global and local modules, we visualized their attention levels in Figure \ref{fig:heatmap_pxrd}. Specifically, we used Grad-CAM \citep{selvaraju2017grad,zeiler2014visualizing,pope2019explainability} for the local module's focus and analyzed the self-attention matrix for the global module. Figure \ref{fig:heatmap_pxrd} shows the local module focuses on local variations between peaks (red dashed box), while the global module attends to patches with high peaks (orange dashed box). Thus, the two modules capture distinct fine-grained information from the PXRD pattern, which is key to XRDecoupler's ability to learn detailed representations and achieve superior performance.

\begin{table}[t]
    \centering
    \scalebox{0.7}{
        \begin{tabular}{l c c c }
            \toprule[2pt]
            \multicolumn{1}{l}{\multirow{2}{*}{Method}} & \multicolumn{3}{c}{Accuracy (\%)} \\
            \cmidrule{2-4}
            & Top-1 & Top-2 & F1 Score\\
            \midrule
            MLP        & 15.50  {($-$14.10)}  & 21.4 {($-$23.20)}  & 8.5 {($+$0.8)}  \\
            CNN        & 29.60 {($+$0.00)}  & 44.60 {($+$0.00)}  & 7.70 {($+$0.00)}  \\
            \midrule
            NoPoolCNN  & 30.40 {($+$0.8)}   & 41.60 {($-$3.00)}   & 15.90 {($+$8.2)} \\
            RCNet          & 41.70 {($+$12.10)} & 52.40 {($+$7.80)} &   19.40 {($+$11.70)}   \\
            XRDMamba        & {54.50} {($+$24.90)} & {64.70} {($+$20.10)} & {24.10} {($+$16.40)}  \\
            \midrule
            \rowcolor{gray!10} XRDecoupler         & \textbf{60.22} {($+$30.62)} & \textbf{70.45} {($+$25.85)} & \textbf{29.09} {($+$21.39)}  \\
            \bottomrule[2pt]\vspace{-0.3cm}
        \end{tabular}}
        \tabcaption{Generalization analysis on the inorganic dataset with SOTA methods. \textbf{Bold} indicates the best performance while {\textbf{($+$)}} and {\textbf{($-$)}} indicate the the relative gain with CNN.}
    \label{tbl:inorganic}
\end{table}

\subsubsection*{$\blacktriangleright$ Generalization Analysis.}  We conducted a generalization test using 8,000 inorganic crystal data obtained from \citep{salgado2023automated}, which encompasses 178 space group categories. According to the theoretical logic of symmetry classification, inorganic crystals and MOF crystals are analogous, differing primarily in their building units, types of chemical bonds, topological structures, and pore structures. Therefore, inorganic crystals represent out-of-domain data for us, and the model's performance on this data serves as a good measure of its generalization capability. As shown in Table \ref{tbl:inorganic}, our method outperforms the SOTA methods, which indicates that XRDecoupler significantly enhances generalization performance on out-of-domain data. 

\subsubsection*{$\blacktriangleright$ Further analysis.} Please see \textbf{\underline{Supplement F}} to find more analysis, including ablation study, crystal scale adaptability, and case analyses.

\section{Conclusion}
In this paper, we present the XRDecoupler framework, which efficiently determines the symmetry of crystalline materials. This work advances the application of deep learning in crystalline materials analysis, providing a novel methodology for symmetry identification.

\bibliography{aaai2026}

\clearpage
\appendix
\twocolumn[ 
  \centering 
  {\huge\textbf{Supplementary Material}} \\ 
    \vspace{1em} 
    {\large\textbf{Rethinking Crystal Symmetry Prediction: A Decoupled Perspective}}
  \vspace*{1.5em} 

  \begin{quote} 
    \noindent 
    The content of the \textbf{Supplement} is summarized as follows: 
\begin{enumerate}
    \itemsep0em
   \item In Sec.~\ref{asec:d&b}, we provide detailed explanations for the partial definitions and background involved in the method in this paper for easy understanding.
   \item In Sec.~\ref{asec:proof}, we state the proofs of  Proposition 2, Definition 1, and Proposition 3 mentioned in the paper.  
   \item In Sec.~\ref{asec:dataset} and Sec.~\ref{asec:baseline}, we demonstrate the details of datasets and baselines we use in experiments of XRDecoupler.
   \item In Sec.~\ref{asec:id}, we provide implementation details of XRDecoupler.
   \item In Sec.~\ref{asec:exp}, we illustrate more detailed empirical results and analyses of XRDecoupler.
   \item In Sec.~\ref{asec:l&f}, we analyze some of the limitations of XRDecoupler and give directions for future research.  
   \item In Sec.~\ref{asec:rw}, we summarize existing Crystalline Space Group Identification and Superclass Learning methods and explicitly illustrate the novelty of XRDecoupler.
\end{enumerate}
  \end{quote}
  \vspace*{1em} 
] 

%
%
%
%



\urlstyle{rm} 
\def\UrlFont{\rm}  
\frenchspacing  
\setlength{\pdfpagewidth}{8.5in} 
\setlength{\pdfpageheight}{11in} 


\lstset{%
	basicstyle={\footnotesize\ttfamily},
	numbers=left,numberstyle=\footnotesize,xleftmargin=2em,
	aboveskip=0pt,belowskip=0pt,
	showstringspaces=false,tabsize=2,breaklines=true}
\floatstyle{ruled}
\newfloat{listing}{tb}{lst}{}
\floatname{listing}{Listing}

\pdfinfo{
/TemplateVersion (2026.1)
}

\setcounter{secnumdepth}{2} 

\ifdefined\aaaianonymous
    \title{AAAI 2026 Supplementary Material\\Anonymous Submission}
\else
    \title{AAAI 2026 Supplementary Material\\Camera Ready}
\fi
 
\ifdefined\aaaianonymous
\author{
    Anonymous Submission
}
\affiliations{
}
\else
\fi






\section{Definition and Background}\label{asec:d&b}
\subsection{\textbf{Definition of Space Group}}
In crystallography, the space group serves as a classification method for crystal structures, describing the lattice types and symmetry operations that the crystal structure satisfies. For example, the well-known crystal structure of diamond\citep{liopo2018analysis} belongs to the space group \texttt{Fd}$\overline{\texttt{3}}$\texttt{m} (No. 227), where \texttt{F} denotes the face-centered lattice type, \texttt{d} represents the glide plane, $\overline{\texttt{3}}$ indicates threefold rotational inversion symmetry, and \texttt{m} signifies mirror symmetry. Therefore, space group identification is a relatively fine-grained classification task that requires determining both the lattice type and various symmetry operations (including point symmetry, translational symmetry, glide planes, and screw axis combinations), imposing high classification standards. 

\subsection{\textbf{Classification Systems in Crystallography}}
As shown in  \textbf{\underline{Figure \ref{fig:cls_tpye}}}, in this structure, a solid arrow indicates that the former is a superclass of the latter, while a dashed arrow represents the correspondence between different classification systems. This structure illustrates the classification system in crystallography from macro to micro levels. The dashed boxes indicate the superclass system that we used in this study.
\begin{figure}[t!]
    \centering
    \includegraphics[width= 0.7\linewidth]{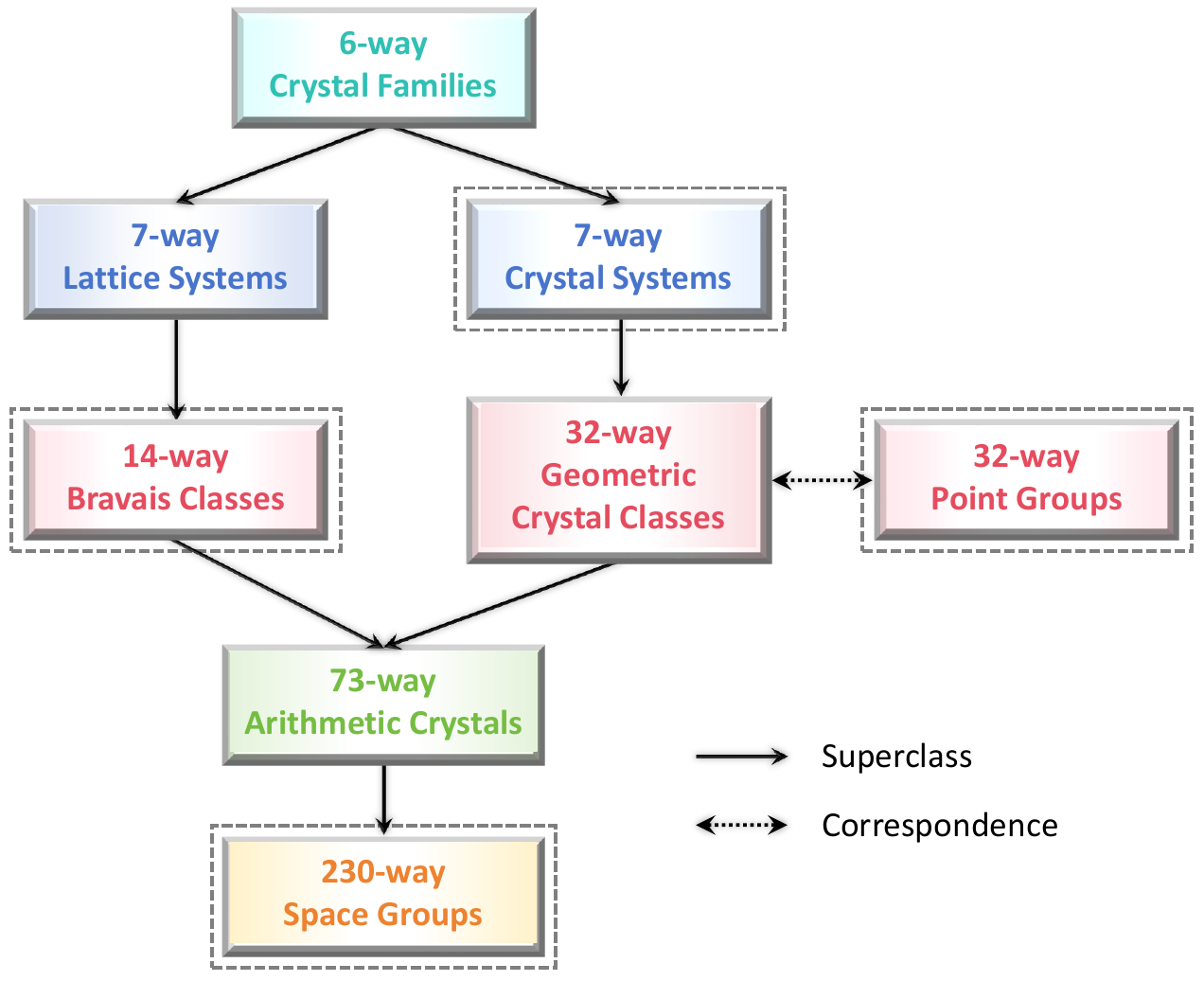}
    \caption{Hierarchical structures and correspondences of classification systems in crystallography. }
    \label{fig:cls_tpye}
\end{figure}

\subsection{\textbf{Background of Superclass Mechanism}}
The superclass, also known as the parent class, is a relative concept in classification systems. It refers to a more general class in a hierarchical classification structure that contains other classes (called subclasses). Specifically, let $S$ represent the set of all classified objects, $C^{1}$ denote one classification method with $n_1$ classes for the set $S$, and $C^{2}$ denote another classification method with $n_2$ classes for the set $S$, such that $S = \bigcup\limits^{i \in [n_1]} C^{1}_i = \bigcup\limits^{j \in [n_2]} C^{2}_j $. If $C^{1}$ is a superclass of $C^{2}$, then for any $j \in [n_2]$, there exists $i \in [n_1]$ such that $C^{2}_j \subseteq C^{1}_i$.

\section{Proof}\label{asec:proof}
\subsection{Proof of Proposition 2}
\setcounter{proposition}{1}

\begin{proposition}[Proposed in paper] 
If there is significant overlap information between labels, maximizing the mutual information between sample representation and truth label will results in the mutual information of sample representation and overlap label being maximized as well, specifically:  
    \begin{equation} 
        \begin{aligned}
           I(E;Y_{truth}) & \uparrow\;\; \Rightarrow \;\;I(E;Y_{other})\uparrow   \\
           & \text{s.t.} \;\; I(M;Y_{truth})\gg  {I(y_{truth};Y_{truth})}.
        \end{aligned} 
    \end{equation}
\end{proposition}
\begin{proof}
According to the chain rule:
\begin{equation} 
    \begin{aligned}
    {I(E;Y_{truth})} = I(E;M)+I(E;y_{truth}|M) \\
    {I(E;Y_{other})} = I(E;M)+I(E;y_{other}|M),
    \end{aligned}
    \label{equ:Ixy}
\end{equation}
where the upper bound of $I(E;y_{truth}|M)$ is $H(y_{truth})$. When $I(M;Y_{truth})\gg  {I(y_{truth};Y_{truth})}$ is satisfied, $H(M)\gg H(y_{truth})$ holds. Thus, maximizing $I(E;Y_{truth})$ depends mainly on maximizing $I(E;M)$, which simultaneously leads to an increase in ${I(E;Y_{other})}$.  
\end{proof}

\subsection{Proof of Definition 1}

\setcounter{definition}{0}
\begin{definition}[Proposed in paper]
    \textbf{Difference $Diff$ in mutual information.}The difference $Diff$ in mutual information between two labels is equal to the mutual information of their non-overlapping parts.
\end{definition}

\begin{proof}
    \begin{equation}
    \begin{split} & I(E;Y_{truth})-I(E;Y_{other}) =\\
    & I(E;M,y_{truth})-I(E;M,y_{other})\\
    &= [H(M,y_{truth})-H(M,y_{truth}|E)]  \\
    & \;\;\;\; - [H(M,y_{other})-H(M,y_{other}|E)] \\
    &= [[H(M)+H(y_{truth})]-[H(M|E)+H(y_{truth}|E)]]\\
    &\;\;\;\;-[[H(M)+H(y_{other})]-[H(M|E)+H(y_{other}|E)]] \\
    &= I(E;y_{truth}) - I(E;y_{other})
    \end{split}
    \end{equation}
\end{proof}

\subsection{Proof of Proposition 3}
\setcounter{proposition}{2}  
\begin{proposition}
    The process of maximizing $Diff$ can be interpreted as maximizing the mutual information between the sample and each structure sub-properties.
\end{proposition}
  
\begin{proof}
\begin{equation}
\begin{aligned} 
Diff =& I(E;y_{truth})-I(E;y_{other})\\
=& I(E;y_{{1}}^{truth},...,y_{{t}}^{truth}) - I(E;y_{{1}}^{other},...,y_{{t}}^{other})\\
=& \sum_{i=1}^{t}{I(E;{y_{i}^{truth}}|{y_{{i-1}}^{truth}},...,{y_{2}^{truth}},{y_{1}^{truth}}}) - \\
&\sum_{i=1}^{t}{I(E;{y_{i}^{other}}|{y_{{i-1}}^{other}},...,{y_{2}^{other}},{y_{1}^{other}}}) \\ 
\end{aligned}
\end{equation} where $\{y^{truth}_i\}$ and $\{y^{other}_i\}$ denote the $i$-th different substructure symmetry properties in $Y_{truth}$ and $Y_{other}$ respectively. Furthermore, according to the assumption of independence, 
\begin{equation}
    \begin{aligned}
        Diff =& \sum_{i=1}^t{I(E;y_{i}^{truth})} - \sum_{i=1}^t{I(E;y_{i}^{other})}\\
        =& \sum_{i=1}^t{[I(E;y_{i}^{truth})-I(E;y_{i}^{other})]} \\
    \end{aligned}
\end{equation}

Therefore, the process of maximizing $Diff$ can be interpreted as maximizing the mutual information between the sample and each structure sub-properties. In other words, we can optimize the model's ability to learn each sub-properties to alleviate the confusion phenomenon.
\end{proof}

\section{Dataset}\label{asec:dataset}
\subsection{Main Dataset}
We use the MOF dataset of over 280,000 metal-organic frameworks (MOFs) \citep{kitagawa2014metal,furukawa2013chemistry,james2003metal,zhou2012introduction} from the Cambridge Crystallographic Data Centre (CCDC) \citep{allen1979cambridge} for our experiments, as \citep{yu2024xrdmamba}. 

Compared to other crystals, MOFs have richer building units, more complex types of chemical bonds, a wider variety of topological structures, and more flexible pore structures. Therefore, the complexity of the MOF data structures is higher than that of most other crystals. Utilizing the MOF dataset as our experimental data not only enriches the knowledge learned by the model but also fully demonstrates the model’s performance. The entire MOF dataset encompasses 225 out of 230 space group categories, as shown in \textbf{\underline{Figure \ref{fig:distribution}}}. Among these, common space groups such as \texttt{P-1}, \texttt{C2/c}, \texttt{P21/c}, \texttt{Pbcn}, and \texttt{P21212} have a larger number of samples, while more complex space groups like \texttt{P4}$_2$\texttt{mc}, \texttt{P4mm}, and \texttt{P6mm} contain only a few samples.

To enhance the credibility of the experimental results, we adopted a half-split approach. Specifically, for a given space group with $c$ samples, we placed $\lfloor c/2 \rfloor$ samples into the training set and $\lceil c/2 \rceil$ samples into the test set. Additionally, we utilized MOF-Balanced as another test dataset. This dataset is a balanced subset sampled from the test set, randomly selecting 10 samples from each space group that has more than 10 samples. The use of MOF-Balanced is primarily to prevent potential biases in the experimental results due to imbalanced testing data in the original MOF dataset. By employing both balanced and imbalanced test datasets, we can observe more objective and nuanced experimental results.

We conducted multiple experiments on both MOF and MOF-Balanced datasets, comprehensively recording the Top-1 accuracy, Top-2 accuracy, Top-5 accuracy, F1-Score, and Recall for each method. These diverse performance metrics are used to reflect the performance gap between XRDecoupler and the baselines.

\begin{figure}[h!]
    \centering
    \includegraphics[width=\linewidth]{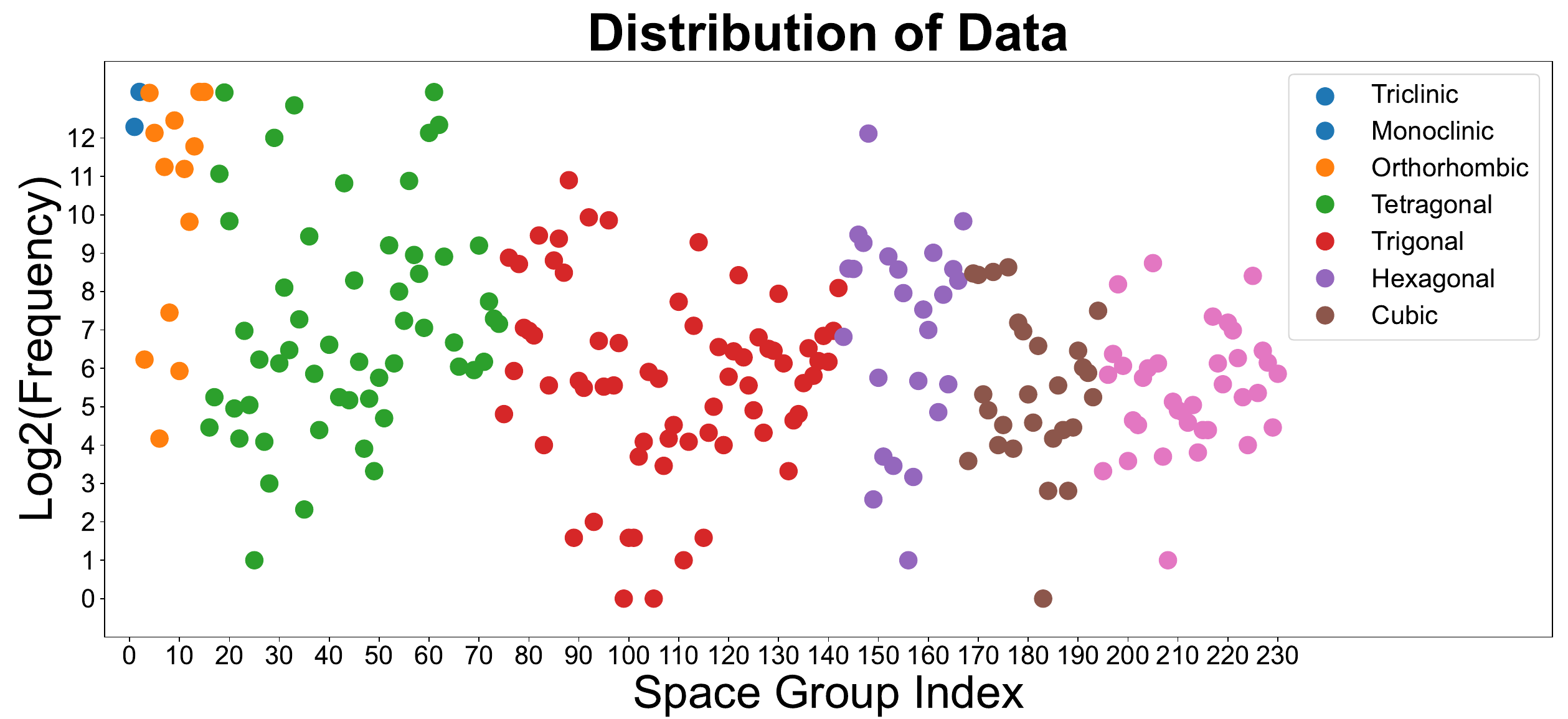}
    \caption{Data distribution of the MOF dataset from CCDC. Due to the significant disparity in sample sizes between different space groups, we applied a logarithmic transformation to the overall distribution. Additionally, different colors represent different crystal systems, revealing that the sample distribution of various space groups within the same crystal system is also quite diverse.}
    \label{fig:distribution}
\end{figure}

\subsection{Other Dataset}
To verify the efficacy of our method on inorganic and small-scale data, we have introduced two additional datasets. The first is CoREMOF \citep{chung2019advances}, a compilation of 14,000 porous, three-dimensional metal-organic framework structures. The second is an inorganic crystal dataset (called InorganicData), comprised of more than 7,000 inorganic crystals sourced from \citep{salgado2023automated}. We conducted a range of tests on these two comparatively small datasets using the same experimental setup as the main dataset. The experimental results further substantiate the effectiveness of our method.

\section{Baseline}\label{asec:baseline}
\begin{itemize}
    \item \textbf{MLP}~\citep{salgado2023automated}: A simple multi-layer perceptron model is used to transform the input sequence and obtain the final categories.
    \item \textbf{CNN}~\citep{salgado2023automated}: This model employs alternating large-kernel 1D convolutional layers and pooling layers to quickly reduce the dimensionality of the input sequence, followed by a classifier for final classification.
    \item \textbf{NoPoolCNN}~\citep{salgado2023automated}: This variant removes the pooling layers from the CNN. The authors proposed that eliminating pooling layers helps mitigate information compression in the data, thereby better preserving contextual information.
    \item \textbf{RCNet}~\citep{chen2024crystal}: Building on CNN, this model introduces residual connections, replacing the stacking of conventional convolutional layers with stacks of residual blocks, further optimizing the network structure.
    \item \textbf{XRDMamba}~\citep{yu2024xrdmamba}: This model introduces a new perspective on PXRD pattern and innovatively incorporates the Mamba model into the space group classification task for the first time.
\end{itemize}
 
\section{Implementation Details}\label{asec:id}
We implemented our method using PyTorch and trained our model on an 80G NVIDIA A100 GPU. During training, we set both the local and global representation sizes to 32 dimensions. In the global module, we configured each patch size to 10, with the peak intensity features $e_{intensity}$ set to 24 dimensions and the learnable features $e_{learnable}$ set to 8 dimensions. We defined the number of attention blocks to be 4, with a hidden layer dimension of 128 during the feed-forward stage.

We trained our model for 150 epochs using the AdamW optimizer, with a weight decay of $2 \times 10^{-5}$ and an initial learning rate of $5 \times 10^{-4}$. For the first twenty epochs, the learning rate remained constant, while in the subsequent 130 epochs, we gradually reduced the learning rate to $1 \times 10^{-6}$ using a cosine annealing schedule.

\section{More Analysis}\label{asec:exp}

\subsection{Other Benchmark Result}
\textbf{\underline{Table \ref{tbl:CoREMOF}}} and \textbf{\underline{\ref{tbl:InorganicData}}} presents the results of benchmark experiments conducted on two datasets: CoREMOF and InorganicData. We can observe that the XRDecoupler enhances accuracy and F1-score by more than 4\% on CoREMOF and 4-5\% on InorganicData. 
 The experimental results suggest that XRDecoupler exhibits superior performance when compared to the baseline on both datasets. It is important to note, however, that both data sets are relatively small and do not cover a comprehensive range of crystals. Consequently, when faced with larger data sets, such as MOF from CCDC, XRDecoupler is likely to deliver even more significant performance enhancements.

\begin{table}[h!]
    \begin{center}
    \caption{Evaluation on CoREMOF dataset with state-of-the-art methods. \textbf{Blod} indicates the best performance.}
    \label{tbl:CoREMOF}
    \scalebox{0.9}{
        \begin{tabular}{l  c c c }
            \toprule[2pt]
            \multicolumn{1}{l}{\multirow{2}{*}{Method}} & \multicolumn{3}{c}{Accuracy (\%)} \\
            \cmidrule{2-4}
            & Top-1 & Top-2 & F1 Score\\
            \midrule
            MLP      & 13.31    & 22.36  & 3.01  \\
            CNN~        & 20.65  & 36.97  & 0.66  \\
            \midrule
            NoPoolCNN~  & 21.89   & 38.86   & 2.75 \\
            RCNet~           & 45.54 & 57.73 &   30.12   \\
            XRDMamba~         & 37.67 & 51.04 & 23.72  \\
            \midrule
            \rowcolor{gray!10} XRDecoupler         & \textbf{49.73}  & \textbf{60.54} & \textbf{34.01}   \\
            \bottomrule[2pt]
        \end{tabular}
    }
    \end{center}
\end{table}

\begin{table}[h!]
    \begin{center}
    \caption{Evaluation on InorganicData dataset with state-of-the-art methods. \textbf{Blod} indicates the best performance.}
    \label{tbl:InorganicData}
    \scalebox{0.9}{
        \begin{tabular}{l  c c c }
            \toprule[2pt]
            \multicolumn{1}{l}{\multirow{2}{*}{Method}} & \multicolumn{3}{c}{Accuracy (\%)} \\
            \cmidrule{2-4}
            & Top-1 & Top-2 & F1 Score\\
            \midrule
            MLP & 21.46    & 25.73  & 13.20  \\
            CNN  & 24.28  & 44.23  & 0.55  \\
            \midrule
            NoPoolCNN & 26.48   & 47.02   & 0.86 \\
            RCNet & 38.37 & 56.92 &   22.46   \\
            XRDMamba & 46.20 & 61.32 & 33.79  \\
            \midrule
            \rowcolor{gray!10} XRDecoupler         & \textbf{51.08}  & \textbf{63.90} & \textbf{39.00}   \\
            \bottomrule[2pt]
        \end{tabular}
    }
    \end{center}
\end{table}
 
\begin{figure*}[h!]
    \centering
    \includegraphics[width=\linewidth]{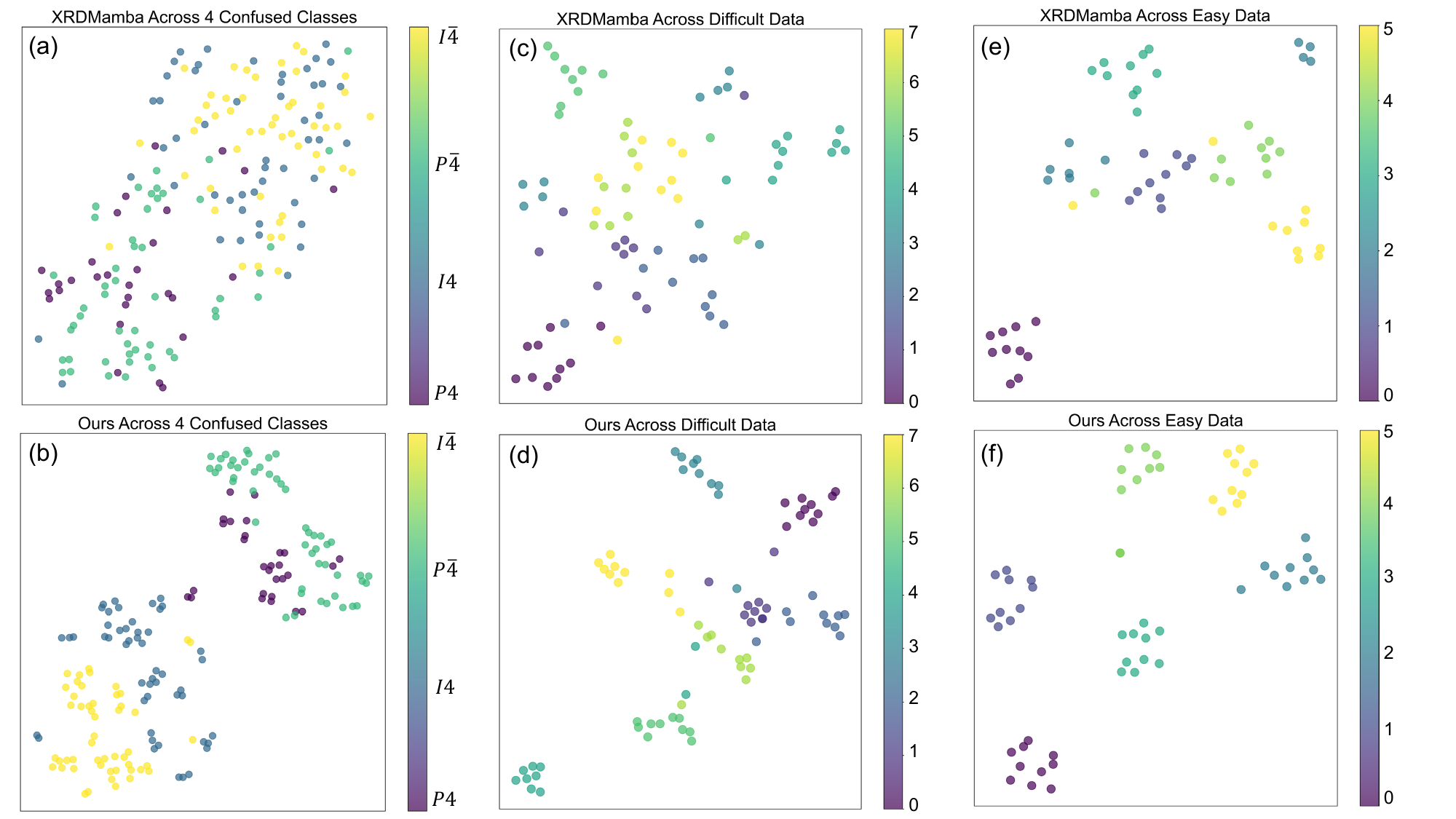}
    \caption{T-SNE Visualization Analysis of XRDMamba and XRDecoupler in the test set. We provide a comparison of confused classes, difficult classes, and easy classes.}
    \label{fig:tsne6}
\end{figure*}

\subsection{T-SNE Visualization Analysis of Confusion Phenomena}
Figures \ref{fig:tsne6}(a) and \ref{fig:tsne6}(b) represent the aforementioned four easily confused categories: \texttt{P4}, \texttt{I4}, \texttt{P}$\overline{4}$, and \texttt{I}$\overline{4}$. From Figure \ref{fig:tsne6}(a), it is evident that XRDMamba does not achieve good fitting for these four categories, showing a significant overlap where \texttt{P4} is mixed with \texttt{P}$\overline{4}$ and \texttt{I4} is mixed with \texttt{I}$\overline{4}$, indicating a serious confusion problem. In contrast, the representations obtained by our method in Figure \ref{fig:tsne6}(b) exhibit clear distinctions across these four categories. Although perfect differentiation has not yet been achieved for some individual samples, the confusion phenomenon has been significantly improved compared to previous methods.

Next, we sampled some space group categories that performed poorly and well on the XRDMamba model, and conducted further comparisons on these samples. Figures \ref{fig:tsne6}(c) and \ref{fig:tsne6}(d) display 8 space group categories that performed very poorly on XRDMamba. As seen in Figure (c), there is significant confusion among several classes in the central region. After introducing our method, although some samples remain difficult to distinguish, there has been a notable improvement in the overall representation space compared to before.  Figures \ref{fig:tsne6}(e) and \ref{fig:tsne6}(f) show 6 categories of samples that performed well on XRDMamba. From Figure \ref{fig:tsne6}(e), we can observe that their representations in XRDMamba already exhibit a certain level of distinction. For these samples, the representations obtained using our method (shown in Figure \ref{fig:tsne6}(f)) achieve nearly perfect discriminability.

\subsection{Crystal Scale Adaptability}
We present the predictive performance of our method on crystals of different scales in \textbf{\underline{Figure \ref{fig:big_coord_sample}}}. It can be found that XRDecoupler can accurately find the correct space group type for different scale crystal materials belonging to different space groups, especially for crystals with more than 500 atoms. This indicates the good adaptability of XRDecoupler at the crystal scale.

\begin{figure*}[h!]
    \centering
    \includegraphics[width=\linewidth]{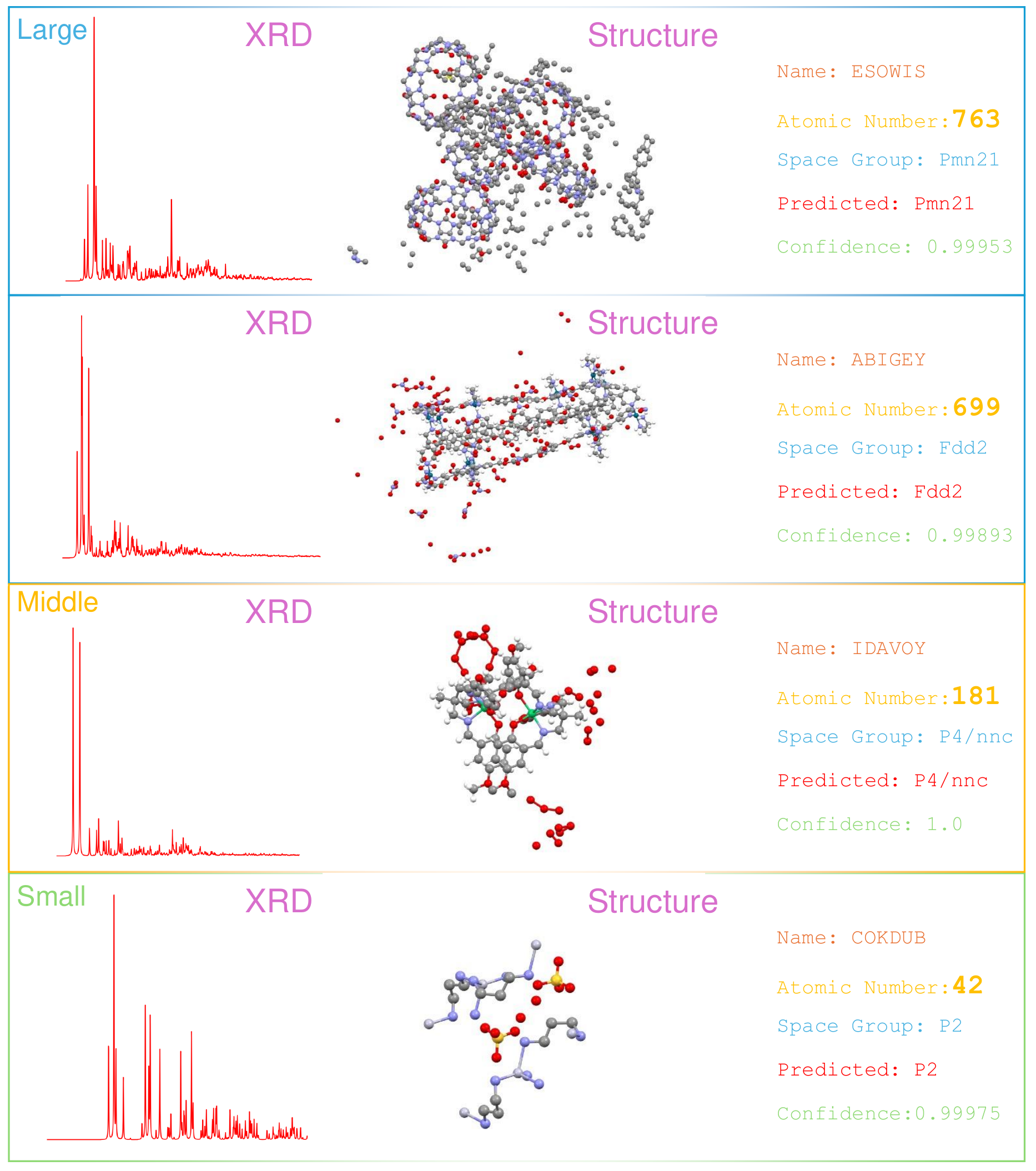}
    \caption{The performance of our method on molecules at different scales.}
    \label{fig:big_coord_sample}
\end{figure*}

\subsection{Ablation Study}
To validate the effectiveness of the Superclass-Guided optimization (SG) and Hierarchical PXRD Pattern Learning (HPSL) modules, we conducted experiments by replacing the encoder model and removing SG. As shown in \textbf{\underline{Table \ref{tbl:ablation_study}}}, when HPSL was used to replace XRDMamba as the encoder (without the SG), we achieved accuracy improvements of 5.81\% and 7.04\% on the MOF and MOF-Balanced test sets, respectively. Other metrics on the MOF test set improved by 2\% to 6\%, while those on the MOF-Balanced test set increased by 7\% to 8\%. These results demonstrate that our proposed HPSL effectively understands PXRD Pattern data, captures detailed information, and encodes robust representations.

Next, we introduced the superclass-guided optimization method on both XRDMamba and HPSL. When this method was added to XRDMamba, accuracy on the MOF and MOF-Balanced test sets improved by 1.9\% and 2.53\%, respectively, with other metrics also increasing by 1\% to 4\%. After introducing the SG method to HPSL, the accuracy on the MOF and MOF-Balanced test sets improved by 2.08\% and 3.13\%, respectively. Thus, it is clear that the introduction of superclasses further enhances the model's learning capability, demonstrating superior performance.

\begin{table*}[h!]
    \centering
    \caption{Ablation study of the effectiveness of SuperClass-guided (\textit{SG}) optimization mechanism and Hierarchical PXRD Pattern Learning (HPSL). $\Delta$ represents the performance improvement before and after using \textit{SG}.}
    \label{tbl:ablation_study}
    \scalebox{0.7}{
        \begin{tabular}{c c ccccc c ccc }
            \toprule[2pt]
            \multicolumn{1}{c}{\multirow{2}{*}{Method}} & \multicolumn{1}{l}{\multirow{2}{*}{\textit{w. SG}}} & \multicolumn{5}{c}{MOF} & & \multicolumn{3}{c}{MOF-Balanced} \\
            \cmidrule{3-7}\cmidrule{9-11}
            && Acc@Top-1 & Acc@Top-2 & Acc@Top-5 & F1 Score & Recall && Acc@Top-1 & Acc@Top-2 & Acc@Top-5 \\
            \midrule
            \multicolumn{1}{c}{\multirow{2}{*}{XRDMamba}}    & \ding{56} & 72.20 & 85.20 & 93.42 & 47.59 & 46.00 && 48.70 & 61.70 & 74.83 \\
                                                             & \ding{52} & 74.10 & 86.10 & 93.98 & 49.96 & 48.35 && 51.23 & 64.53 & 78.37 \\
            $\Delta$                        &     -     & {$+$1.90} & {$+$0.90} & {\color{black}$+$0.56} & {$+$2.37} & {$+$2.35} && {$+$2.53} & {$+$2.83} & {$+$3.54} \\
            \midrule
            \multicolumn{1}{c}{\multirow{2}{*}{HPSL}} & \ding{56} & 78.01 & 89.37 & 95.83 & 53.63 & 52.17     && 55.74 & 69.08 & 81.01 \\
                                                             & \ding{52} & 80.09 & 90.11 & 96.26 & 56.72 & 55.18 && 58.87 & 72.42 & 85.22 \\
            $\Delta$                        &     -     & {$+$2.08} & {$+$0.74} & {$+$0.43} & {$+$3.09} & {$+$3.01} && {$+$3.13} & {$+$3.34} & {$+$4.21} \\
            \bottomrule[2pt]
        \end{tabular}
    }
\end{table*}

\subsection{Challenge Case Study}
In our analysis of the experimental results, we particularly focused on the model's performance on complex crystal structures with limited sample sizes. To this end, we specifically selected some challenging samples from the test set for analysis and evaluation. As shown in \textbf{\underline{Figure \ref{fig:case_study}(top)}}, we were excited to find that although the \texttt{P6mm} space group has only one sample in the training set, our model was still able to correctly predict the \texttt{DUJWUA10} sample in the test set. This indicates that our model retains a certain structural analysis capability even for complex structures.

However, it is also noteworthy that the model still exhibits prediction biases for many complex crystals. For example, in the case of \texttt{DECKAA} shown in \textbf{\underline{Figure \ref{fig:case_study}(bottom)}}, the model successfully predicted 2 superclasses but made an error in predicting the point group. This suggests that while our model is close to making correct predictions for these complex crystals, it still falls short, and we attribute this bias to the scarcity of samples. Due to the inherent issue of sample scarcity in complex crystal structures, even though we introduced superclasses to guide the samples, this subset inevitably gets overshadowed by the more abundant samples from other space groups during training, leading to incorrect judgments by the model. Enhancing the model’s focus on these challenging samples will therefore be one of our primary areas of focus in the next phase of our work.

\begin{figure}[h!]
    \centering
    \includegraphics[width=0.5\linewidth]{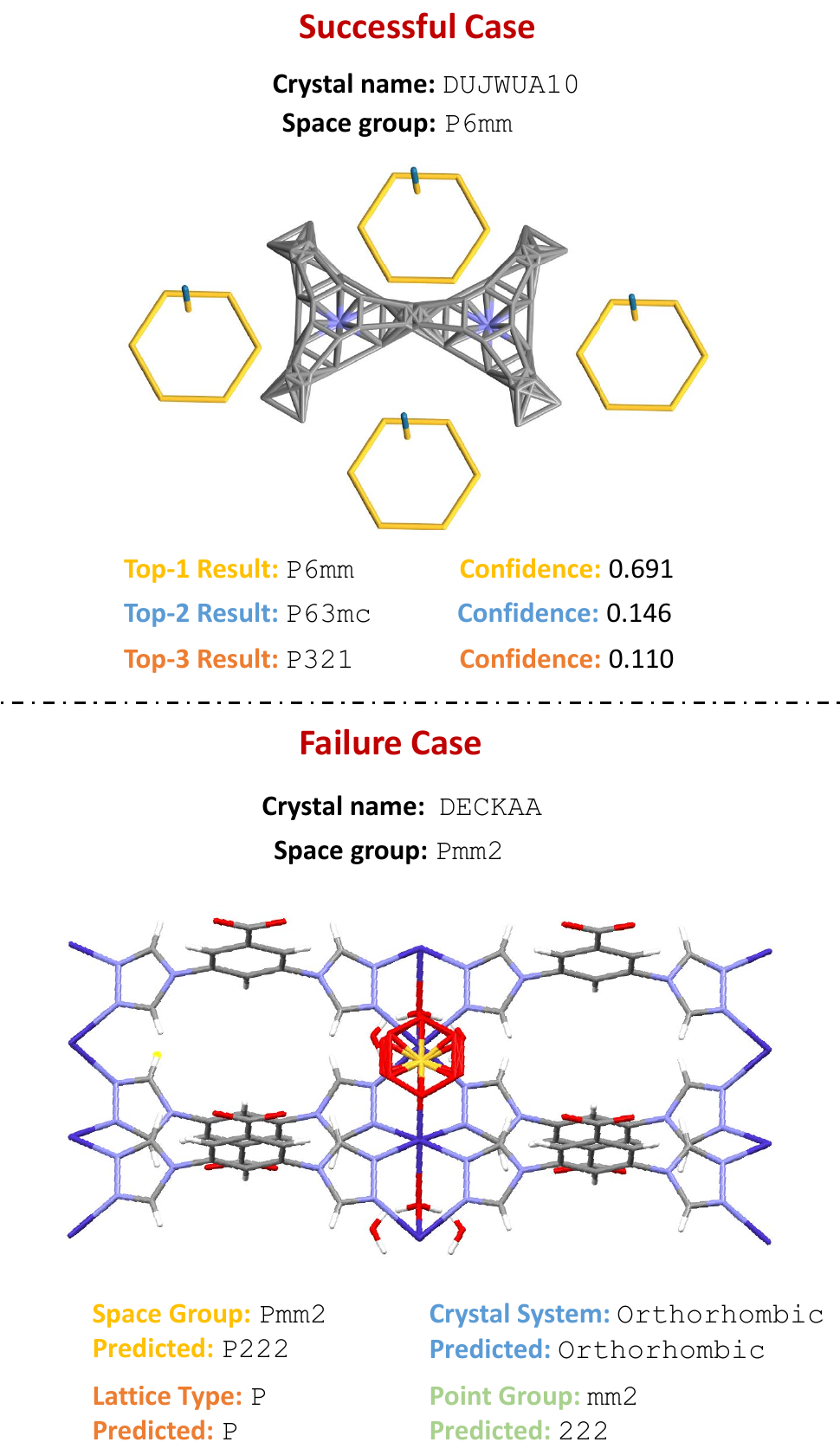}
    \caption{The classification results and relevant information for two crystal cases are presented, including a successful case (\texttt{DUJWUA10}) and a failed case (\texttt{DECKAA}). This figure highlights the differences in the classification model's performance across various space group predictions by comparing the successful and failed cases.}
    \label{fig:case_study}
\end{figure}

\section{Limitations and Future Works}\label{asec:l&f}
Here, we further analyze the limitations of XRDecoupler and propose ways to address its limitations in future work:

\begin{itemize}
    \item In this study, we innovatively introduced the use of superclasses to guide space group learning. However, due to the large number of superclasses associated with space groups, enumerating all possible combinations for experimentation is a significant challenge. In our current work, we only selected key superclasses. In future work, we will continue to explore and analyze whether there are more optimal combinations of superclasses that can maximize the effectiveness of the superclass guidance.
    \item In the space group prediction task, due to the high cost of obtaining crystal data, existing works have primarily trained models on a small subset of fine-grained crystal data, lacking a unified model capable of predicting all types of crystals. Although our work focuses on the complex and challenging metal-organic frameworks (MOFs), our generalization studies reveal that even though inorganic crystal structures are significantly simpler than MOF structures, their performance on out-of-domain data is still lower than that on in-domain data. Therefore, in future work, we aim to achieve unified predictions by incorporating more crystal data, and training a model that demonstrates good predictive performance across various types of crystal data.
    \item Existing crystal data inevitably faces the issue of data imbalance. Case studies have shown that the model performs poorly on many complex space groups with limited sample sizes. Improving the model's performance on this subset of data will also be one of our key focuses in future work. 
\end{itemize}

\section{Related Works}\label{asec:rw}
\subsection{Crystalline Space Group Identification}
Quickly identifying the space groups of crystal materials is crucial for predicting their three-dimensional structures. Early works\citep{werner1985treor} primarily relied on computational methods for complex derivations. Recent studies have attempted to apply machine learning and deep learning techniques to analyze diffraction data, aiming to improve the efficiency and accuracy of space group recognition. For example, Park et al.\citep{park2017classification} introduced convolutional neural networks (CNNs) for space group identification, but their trained model was only tested on two experimental spectra, with incorrect predictions on one of them. Subsequently, Vecsei et al.\citep{ziletti2018insightful,oviedo2019fast,vecsei2019neural,dong2021deep} also employed CNNs, training similar deep models that achieved decent results. However, their trained models were either used to identify a subclass within crystal symmetry categories or only involved data from a single domain within material databases. In contrast, Salgado et al.\citep{salgado2023automated} utilized a large amount of crystal data that nearly covered all types of space groups, proposing a model called NPCNN. While this method achieved comprehensive predictions, its recognition accuracy was not satisfactory. Following this, RCNet\citep{chen2024crystal} customized crystal structure categories to reduce recognition difficulty and incorporated residual structures into convolutional networks, further enhancing recognition performance. XRDMamba\citep{yu2024xrdmamba} was the first to integrate chemical knowledge into the model, presenting a new perspective for modeling PXRD spectra and incorporating the Mamba\citep{gu2023mamba,gu2023modeling} architecture, which yielded good results on complex MOF datasets.

However, the aforementioned methods did not fully leverage relevant chemical rules in their model design, lacking targeted designs related to space group recognition tasks, and they suffered from significant confusion between classes. To address these issues, our proposed method comprehensively considers the confusion challenges faced in space group recognition from a chemist's perspective. By treating multiple recognition approaches as superclasses to make collective decisions, we effectively mitigate the performance decline caused by confusion and enhance the model's generalization ability.

\subsection{Superclass Learning}
Superclass learning aims to enhance traditional deep learning by incorporating superclass labels as intermediate supervision. A superclass-guided network can effectively integrate high-level semantic information into the model, improving its representational capacity. Recently, superclass learning has garnered increasing attention from researchers, many of whom have introduced superclasses into their work and made adaptive designs for specific tasks, resulting in significant performance improvements. For example, in image classification and object detection, Dehkordi et al.\citep{dehkordi2022multi}  proposed a two-stage multi-expert framework that utilizes superclass supervision. In the context of imbalanced learning, Zhou et al.\citep{zhou2018deep} alleviated data imbalance by clustering original classes into a balanced superclass space. Superdisco et al.\citep{du2023superdisco} innovatively introduced the concept of superclasses into graph neural networks to learn the relationships between superclasses and samples. Additionally, in synthetic aperture radar (SAR) target classification tasks, MSA-SCNN\citep{wang2022sar} incorporated superclass labels to enhance feature differentiation among classes, thereby improving the model's generalization ability. 

The benefits of incorporating superclasses have been well-documented across various studies. Given the unique classification system of space groups and the compatibility of the superclass concept, we have also innovatively introduced superclasses into space group recognition tasks. By guiding the model with superclasses, we enable it to learn more detailed structural knowledge about crystals, significantly enhancing its performance in space group recognition.

 
\end{document}